\documentclass[preprint,12pt,authoryear]{elsarticle}

\usepackage{amssymb}
\usepackage{amsmath}
\usepackage{url}
\usepackage{subfigure}
\usepackage[T1]{fontenc}
\usepackage{algorithm}%
\usepackage{algpseudocode}%
\usepackage{amsthm}%
\usepackage{booktabs}%
\usepackage{multirow}%
\usepackage{hyperref}
\usepackage{graphicx}

\newtheorem{theorem}{Theorem}

\journal{Neural Networks}

\begin{document}

\begin{frontmatter}



\title{Efficient Adaptive Label Refinement for Label Noise Learning} 


\author[1,2,3]{Wenzhen Zhang} 

\author[4]{Debo Cheng} 

\author[1,2,3]{Guangquan Lu}

\author[5]{Bo Zhou} 

\author[6]{Jiaye Li} 

\author[1,2,3]{Shichao Zhang\corref{cor1}}
\ead{zhangsc@mailbox.gxnu.edu.cn}
\affiliation[1]{organization={School of Computer Science and Engineering},
            addressline={Guangxi Normal University}, 
            city={Guilin},
            postcode={541004}, 
            state={Guangxi},
            country={China}}
            
\affiliation[2]{organization={Guangxi Key Lab of Multi-Source Information Mining \& Security},
            addressline={Guangxi Normal University}, 
            city={Guilin},
            postcode={541004}, 
            state={Guangxi},
            country={China}}
            
\affiliation[3]{organization={Key Lab of Education Blockchain and Intelligent Technology, Ministry of Education},
            addressline={Guangxi Normal University}, 
            city={Guilin},
            postcode={541004}, 
            state={Guangxi},
            country={China}}
            
\affiliation[4]{organization={UniSA STEM},
            addressline={University of South Australia}, 
            city={Adelaide},
            postcode={5095}, 
            state={SA},
            country={Australia}}

\affiliation[5]{organization={Guangxi Collaborative Innovation Center of Modern Sericulture and Silk},
            addressline={Hechi University}, 
            city={Hechi},
            postcode={546300}, 
            state={Guangxi},
            country={China}}

\affiliation[6]{organization={the State Key Laboratory of Blockchain
and Data Security},
            addressline={Zhejiang University}, 
            city={Hangzhou},
            postcode={310027}, 
            state={Zhejiang},
            country={China}
            }

\cortext[cor1]{Corresponding author.}

\begin{abstract}
Deep neural networks are highly susceptible to overfitting noisy labels, which leads to degraded performance. Existing methods address this issue by employing manually defined criteria, aiming to achieve optimal partitioning in each iteration to avoid fitting noisy labels while thoroughly learning clean samples. However, this often results in overly complex and difficult-to-train models. To address this issue, we decouple the tasks of avoiding fitting incorrect labels and thoroughly learning clean samples and propose a simple yet highly applicable method called Adaptive Label Refinement (ALR). First, inspired by label refurbishment techniques, we update the original hard labels to soft labels using the model's predictions to reduce the risk of fitting incorrect labels. Then, by introducing the entropy loss, we gradually `harden' the high-confidence soft labels, guiding the model to better learn from clean samples. This approach is simple and efficient, requiring no prior knowledge of noise or auxiliary datasets, making it more accessible compared to existing methods. We validate ALR's effectiveness through experiments on benchmark datasets with artificial label noise (CIFAR-10/100) and real-world datasets with inherent noise (ANIMAL-10N, Clothing1M, WebVision). The results show that ALR outperforms state-of-the-art methods.
\end{abstract}




\end{frontmatter}



\section{Introduction}\label{sec1}
The adoption of deep neural networks (DNNs) has significantly advanced the field of classification tasks. However, training high-performance classification models often relies on extensive and accurately annotated datasets, which are challenge and often impractical to obtain in practical applications \citep{li2024rethinking, zhu2024robust, zhou2021learning}. In practice, datasets are often annotated using approaches such as crowdsourcing \citep{yu2018learning} or online searches \citep{liu2011noise}. Although these approaches can collect large amounts of data, they inevitably introduce label noise \citep{zhang2024gradient}. Existing research has shown that DNNs have a strong memorization capability, which allows them to learn noisy labels in data \citep{zhang2017understanding}. This can lead to overfitting to incorrect labels during training, resulting in decreased performance. Thus, designing algorithms robust to label noise is essential for effectively utilizing large-scale datasets, when label quality is low.

Recently, significant attention has been focused on developing and refining methods to address label noise, resulting in notable progress \citep{zhang2017understanding, liu2020early, shi2023self}. The early learning phenomenon observed when training on datasets containing mislabeled instances has been highlighted in numerous studies. Specifically, DNNs initially memorize simpler (clean) samples \citep{arpit2017closer, liu2020early}, while samples with larger loss values are often indicative of mislabeled instances. Based on this insight, several approaches have been proposed, including assigning smaller weights to noise samples \citep{shu2019meta} or applying regularization techniques to minimize the negative effects of label noise \citep{krogh1991simple, srivastava2014dropout, xia2020robust}. Nonetheless, the highly nonlinear nature of typical classification loss functions may still lead these methods to inadvertently fit incorrect labels \citep{nguyen2020self}.

A general approach to addressing this issue is sample selection based on early learning \citep{han2018co, yu2019does}, where noisy labels are identified using loss values or confidence scores, enabling the model to prioritize clean labels for training. Moreover, DivideMix \citep{li2020dividemix} effectively mitigates label noise by integrating sample selection with semi-supervised learning. However, accurately differentiating between clean and noisy samples is critical for maintaining model robustness. These methods are often complex and require meticulous hyperparameter tuning to avoid incorrectly classifying noisy labels as clean, which complicates their deployment in practical applications. Furthermore, some methods rely on unrealistic prerequisites. For instance, Meta-weight requires additional small, clean datasets to assist in training \citep{shu2019meta}, while Forward assumes prior knowledge about the noise distribution \citep{patrini2017making}. Despite these efforts, DNNs still prone to fitting incorrect labels, which remains a significant challenge.

 To effectively reduce the negative impact of label noise, a recent category of label refurbishment approaches has been proposed \citep{lu2023rethinking, szegedy2016rethinking, lu2022selc}, becoming a focal point of research. During training, label refurbishment methods generate soft labels by proportionally mixing hard labels (one-hot labels) with a specific distribution, such as the model’s predicted probability distribution, to serve as fitting targets \citep{lu2023rethinking, gong2024does}. For example, Reed \textit{et al.} proposed Bootstrapping Loss, which updates the target by mixing the original noisy labels with the model’s predicted probability distribution \citep{2015TRAINING}. Consider a classification example with two categories: airplanes and birds. Suppose that we have an image of an airplane that is incorrectly labeled as a bird, with its one-hot label as $[0, 1]$. If the model predicts $[0.5, 0.5]$, the refurbished soft label is calculated using the formula: $[\lambda \times 0 + 0.5 \times (1 - \lambda), \lambda \times 1 + 0.5 \times (1 - \lambda)]$, where $\lambda$ is a hyperparameter with a value in the range $[0, 1)$. Refurbished labels mitigate the impact of label noise, as incorrect labels exert a reduced negative influence on the model after refurbishment and may even be corrected. However, while label refurbishment reduces the model’s tendency to fit noisy labels, it can also lead to the underfitting of clean labels, which are converted similarly from hard labels to soft labels. As illustrated in Figure \ref{fig0} (a) and (b), the t-SNE visualization reveals that clean samples within the same class are not sufficiently clustered, and the separation between different classes is weak. This issue intensifies with a higher noise rate.

\begin{figure}[htbp]
    \centering
    \subfigure[LR-Sym-20\%]{\includegraphics[width=0.49\textwidth]{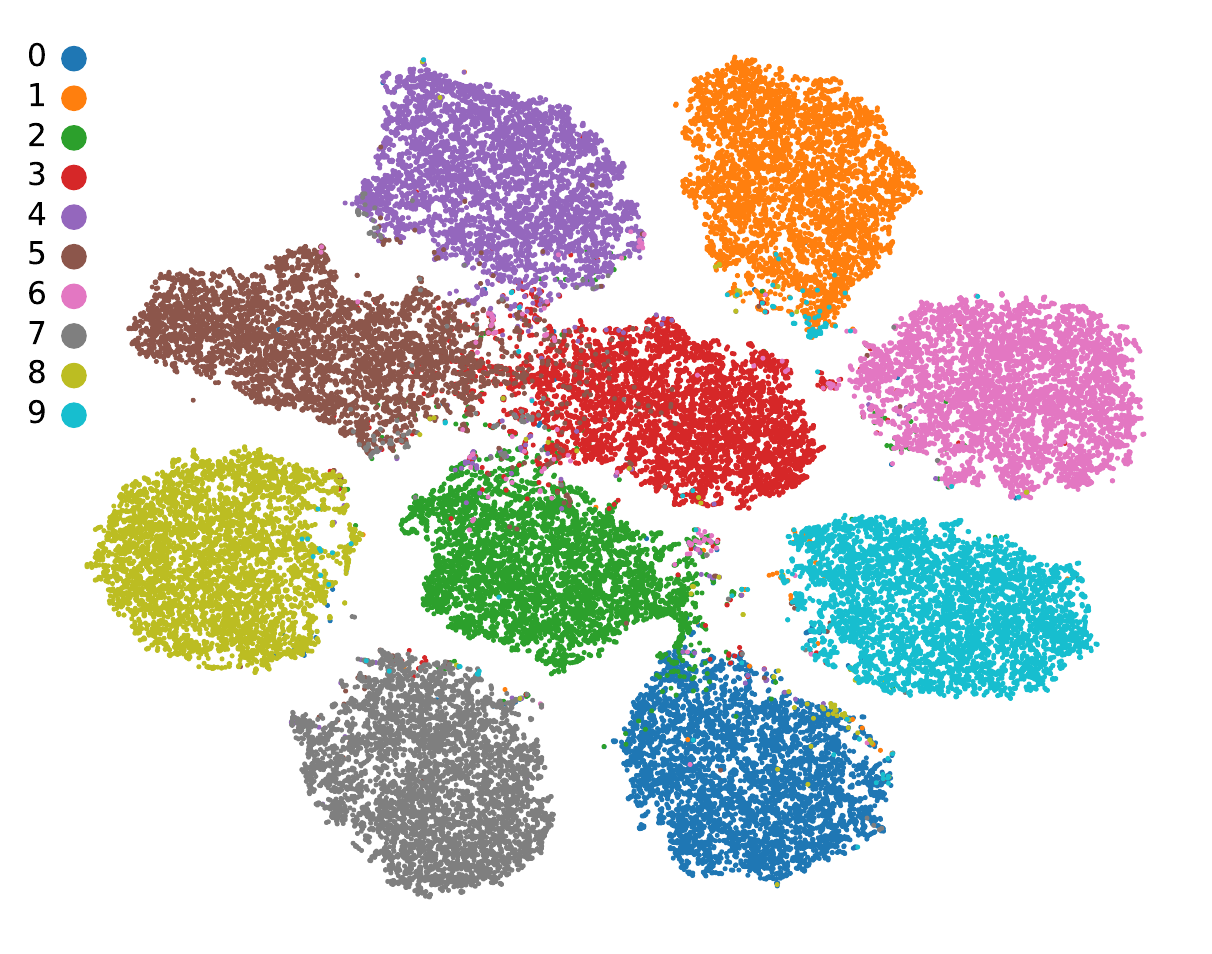}}
    \subfigure[LR-Sym-80\%]{\includegraphics[width=0.49\textwidth]{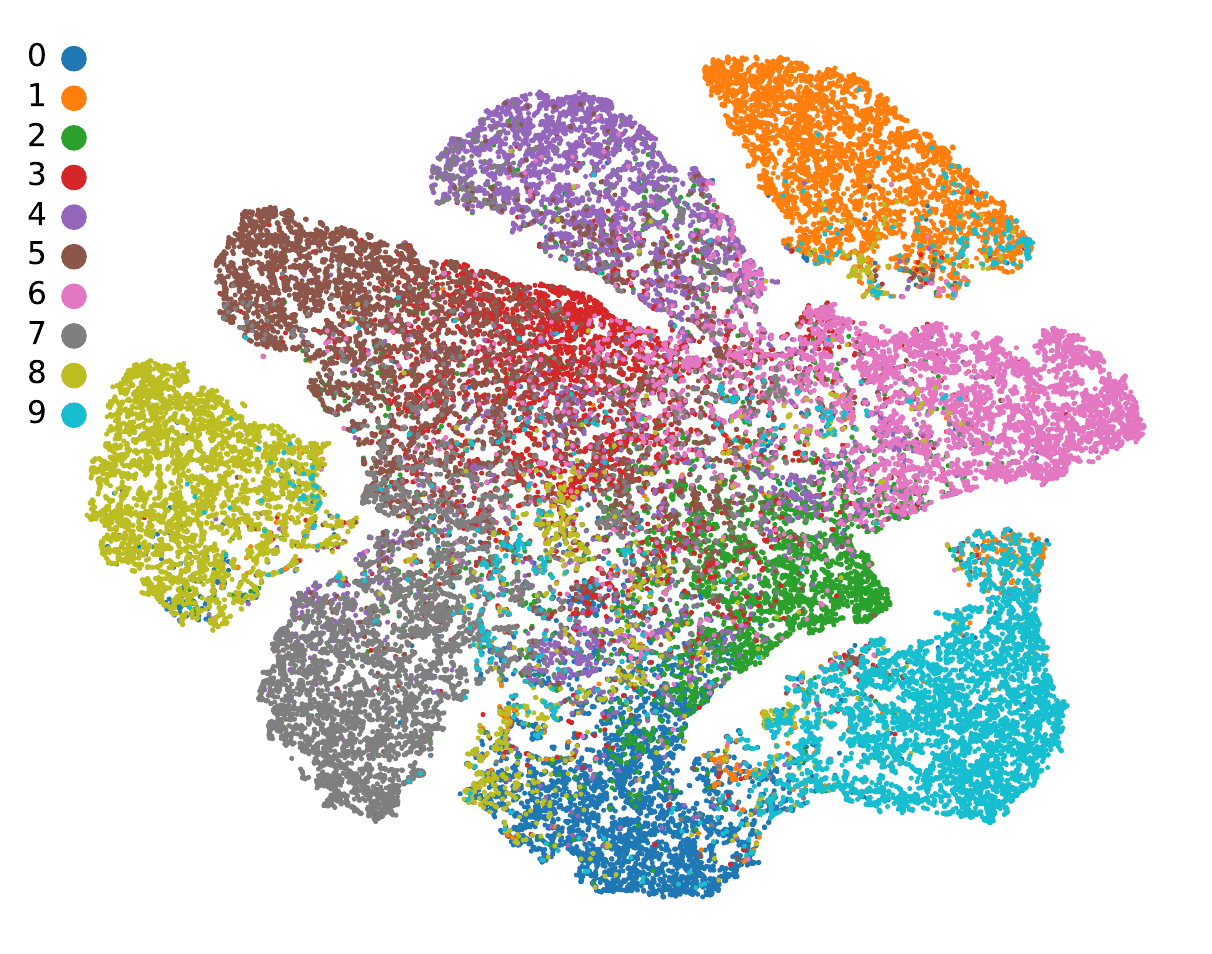}}
    
    \caption{t-SNE visualizations in plots (a) and (b) illustrate the relationships among clean samples in the CIFAR-10 training set after training the model with the pure label refurbishment (LR) method under 20\% and 80\% symmetric label noise.}
    \label{fig0}
\end{figure}

Although sample selection-based methods demonstrate superior performance, accurately distinguishing the side effects of noisy labels (e.g., reliance on complex models and sensitive hyperparameters) poses a significant challenge in practical applications. Although methods like label refurbishment are simple and easy to implement, they may sacrifice specific performance metrics in pursuit of robustness. The primary difficulty with sample selection methods arises from their attempt to simultaneously address two distinct issues through a single operation: distinguishing between clean and noisy labels to thoroughly learn clean labels while avoiding fitting noisy labels. This is like walking a tightrope, where maintaining balance is inherently difficult. However, the limitation of label refurbishment methods is that they focus solely on alleviating a model's tendency toward fitting incorrect labels, often compromising its ability to thoroughly learn clean labels. This raises a key question: \emph{How can we thoroughly learn clean labels while avoiding memorization of noisy labels?}

To tackle this challenge, this paper presents \underline{A}daptive \underline{L}abel \underline{R}efinement (ALR), a simple and highly applicable method that decouples the tasks of thoroughly learning clean labels and avoiding memorization of incorrect labels. The core idea of ALR lies in its implicit handling of clean and noisy labels, focusing on adaptive improvement rather than explicit differentiation. Inspired by SELC \citep{lu2022selc}, ALR employs a temporal integration strategy to generate soft labels. These are created by mixing the original hard labels with the probability distributions derived from the network’s historical predictions, which serve as the fitting target. These soft labels reduce the impact of noisy labels while offering flexibility to handle clean and noisy labels in a unified manner.
To ensure thorough learning of clean labels, we introduce an additional regularization term that addresses the issue of underfitting clean labels. Our method progressively “hardens” high-confidence soft labels during training, enabling the model to focus on increasingly precise and accurate learning objectives, ultimately improving its performance. As the model’s accuracy improves, its predictions become more reliable, allowing it to effectively learn from an increasing number of clean labels. These two processes reinforce each other, leading to steady performance improvements throughout the training iterations.

The main contributions of this study can be outlined as follows:
\begin{itemize}
  \item  We introduce a novel strategy for label refinement that decouples the tasks of thoroughly learning of clean labels and preventing memorizing noisy labels.
  \item We introduce a simple yet effective end-to-end label refinement method, called Adaptive Label Refinement (ALR), to address label noise. ALR requires no prior knowledge of noise distribution or stringent prerequisites, making it easy to implement and widely applicable.
  \item  ALR demonstrates exceptional robustness, achieving state-of-the-art accuracy on multiple artificial and real-world noisy label datasets. Furthermore, it outperforms the performance of traditional cross-entropy loss on clean datasets.
\end{itemize}

\section{Related Work}\label{sec2}
Label noise remains a significant challenge in deep learning for image classification. Numerous approaches have been introduced to address this problem, typically falling into four main categories: robust loss, regularization, sample selection, and label refinement. Below, we provide a brief discussion of these approaches.

\textbf{Robust loss.} Robust loss methods focus on designing loss functions capable of effectively handling label noise. Compared to the Cross-Entropy (CE), the mean absolute error (MAE) has higher robustness to noisy labels but tends to suffer from underfitting \citep{ghosh2017robust}. To address this limitation, subsequent studies have aimed to balance the model’s fitting ability and robustness to enhance performance. For example, Generalized Cross Entropy (GCE) achieves an optimal balance between CE and MAE by adjusting a parameter \citep{zhang2018generalized}. Soft Cross Entropy (SCE) integrates reverse cross-entropy with cross-entropy \citep{wang2019symmetric}. Dynamic Adaptive Loss (DAL) establishes a dynamic balance between the model’s fitting capability and robustness, further enhancing generalization performance \citep{li2023dynamics}. These methods are theoretically well-supported and can be easily integrated with other approaches.

\textbf{Regularization.} Regularization is a powerful strategy to combat label noise by limiting model sensitivity to noisy labels. This is achieved by introducing constraints or adjusting model complexity. Common regularization techniques include Data Augmentation \citep{shorten2019survey}, Weight Decay \citep{krogh1991simple}, and Dropout \citep{srivastava2014dropout}. To specifically address noisy labels, Label Smoothing (LS) improves the model’s generalization by smoothing the original noise labels \citep{szegedy2016rethinking}. Arpit \textit{et al.} demonstrated that regularization can hinder the memorization of noise \citep{arpit2017closer}, showing that models tend to learn clean data before fitting noisy labels. Building on this, ELR introduces a regularization term that utilizes the early learning phenomenon to mitigate noise memorization \citep{liu2020early}. Similarly, CDR proposed a regularization term that deactivates non-critical model parameters during training iterations \citep{xia2020robust}. Sparse regularization encourages the model’s predictions to converge towards a sparse distribution \citep{zhou2021learning}. CTRR further proposes a contrastive regularization function that reduces overfitting to noisy labels by constraining feature representations \citep{yi2022learning}.

\textbf{Sample selection.} This category of approaches focuses on distinguishing clean and noisy samples to prevent models from learning incorrect labels \citep{jiang2018mentornet, gui2021towards}. Meta-Weight-Net additionally uses an online meta-learning method that uses a multilayer perceptron to automatically weight samples \citep{shu2019meta}. However, this approach requires a small clean dataset for auxiliary training \citep{shu2019meta}. Many methods utilize the small loss criterion, which assumes that clean samples generally exhibit lower loss values. For example, Co-Teaching uses two networks, each selecting samples with losses below a specific threshold to train the other network \citep{han2018co}. Co-Teaching+ \citep{yu2019does} refines this idea by relying on the disagreement between two networks for sample selection. However, these methods often require manually set thresholds, making them challenging to apply in practice. To address this, AUM designed the Area Under the Margin statistic, proposing a selection mechanism based on the statistic to identify mislabeled samples \citep{pleiss2020identifying}. 

\textbf{Label refinement.} Label refinement involves mixing the original noisy labels with model predictions to generate updated labels. PENCIL introduced an end-to-end label correction method, which simultaneously updates network parameters and label estimates \citep{yi2019probabilistic}. SELC \citep{lu2022selc} leverages the early learning phenomenon to train a sufficiently good model and subsequently updates labels using temporal ensemble methods. SOP \citep{liu2022robust} models label noise using sparse over-parameterization, implicitly correcting noisy labels. Many label refinement methods incorporate sample selection to first identify clean samples and subsequently correct the noisy labels. For example, M-DYR-H used a two-component beta mixture model to approximate the loss distribution, estimating the probability of label errors and correcting labels based on these probabilities \citep{arazo2019unsupervised}. Similarly, DivideMix \citep{li2020dividemix} employs two networks and applies a dual-component Gaussian mixture model to identify mislabeled samples, which are then treated as unlabeled. It then applies semi-supervised learning \citep{yao2023better} to enhance model robustness.

Unlike previous methods, our ALR method leverages label refinement to prevent overfitting to noisy labels while introducing a regularization term to ensure sufficient learning of clean samples. Specifically, ALR eliminates the need for manual thresholding by dynamically learning thresholds automatically in each iteration. Moreover, our method continuously refines and enhances the learning of clean labels throughout the training process, leading to improved robustness and performance.

\section{Preliminaries \& Problem Setting}\label{sec3}
This section presents the preliminaries of supervised classification with label noise and describes the problem setting.

\subsection{Supervised Classification under Label Noise}\label{subsec31}  
This work addresses a supervised classification task with $K$ classes, where the dataset is corrupted with label noise. The corrupted training set is denoted as $\tilde{\mathcal{D}} = \{(\mathbf{x}_i, \tilde{\mathbf{y}}_i)\}_{i=1}^n$, where $\mathbf{x}_i \in \mathbb{R}^d$ is the $i^{\text{th}}$ input, $\tilde{\mathbf{y}}_i \in (0, 1)^K$ denotes the one-hot encoded vector for the $i^{\text{th}}$ input's (potentially corrupted) class label. Note that $\tilde{\mathbf{y}}_i$ has a certain probability of being incorrectly labeled. 

A DNNs model $\mathcal{N}_{\theta}$, characterized by parameters $\theta$, maps $\mathbf{x}_i$ to $K$-dimensional logits $\mathbf{z}_i=\mathcal{N}_{\theta}(\mathbf{x}_i)$, where $\mathbf{z}_i \in \mathbb{R}^K$. These logits are transformed into predicted conditional probabilities for each class using the softmax function $\mathcal{S}$, resulting in $\mathbf{p}_i = \mathcal{S}(\mathbf{z}_i)$, where $\mathbf{p}_i \in [0,1]^K$. 
We evaluate the empirical risk associated with the model predictions by the cross-entropy loss, which is given by:
\begin{equation}
    \mathcal{L}_{ce} = -\frac{1}{n} \sum_{i=1}^{n} \sum_{k=1}^{K} \tilde{\mathbf{y}}_{i[k]} \log \mathbf{p}_{i[k]},
    \label{eq:ce}  
\end{equation}
where $\mathbf{y}_{i[k]}$ denotes the value corresponding to class $k$ in the label vector for the $\mathbf{x}_i$, and $\mathbf{p}_{i[k]}$ indicates the probability of class $k$ in the model's predicted probability distribution. The gradient of the $\mathcal{L}_{ce}$ w.r.t. the logits $\mathbf{z}_i$ can be expressed as:

\begin{equation}
    \frac{\partial \mathcal{L}_{\text{ce}}}{\partial \mathbf{z}_i} = \mathbf{p}_i - \tilde{\mathbf{y}}_i.
    \label{eq:g_ce}
\end{equation}

When optimizing the model based on the $\mathcal{L}_{ce}$ in Eq.(\ref{eq:ce}),  we noticed that initially, the model tends to fit clean samples before gradually fitting the mislabeled ones. Figures \ref{fig_t&t} (a) and (b) illustrate this phenomenon. During the early phase of training, the model’s performance improves; however, it gradually declines during the later phases as noisy labels dominate the learning process. Several prior works have also reported that DNNs exhibit this behavior when learning with noisy labels, a phenomenon referred to as early learning \citep{zhang2017understanding,liu2020early}. ELR \citep{liu2020early} provides further insights into why the model eventually overfits noisy labels by analyzing the gradients of the cross-entropy loss w.r.t. the logits.

\begin{figure}[htbp]
    \centering
    \subfigure[CIFAR-10 Sym-60\%]{\includegraphics[width=0.49\textwidth]{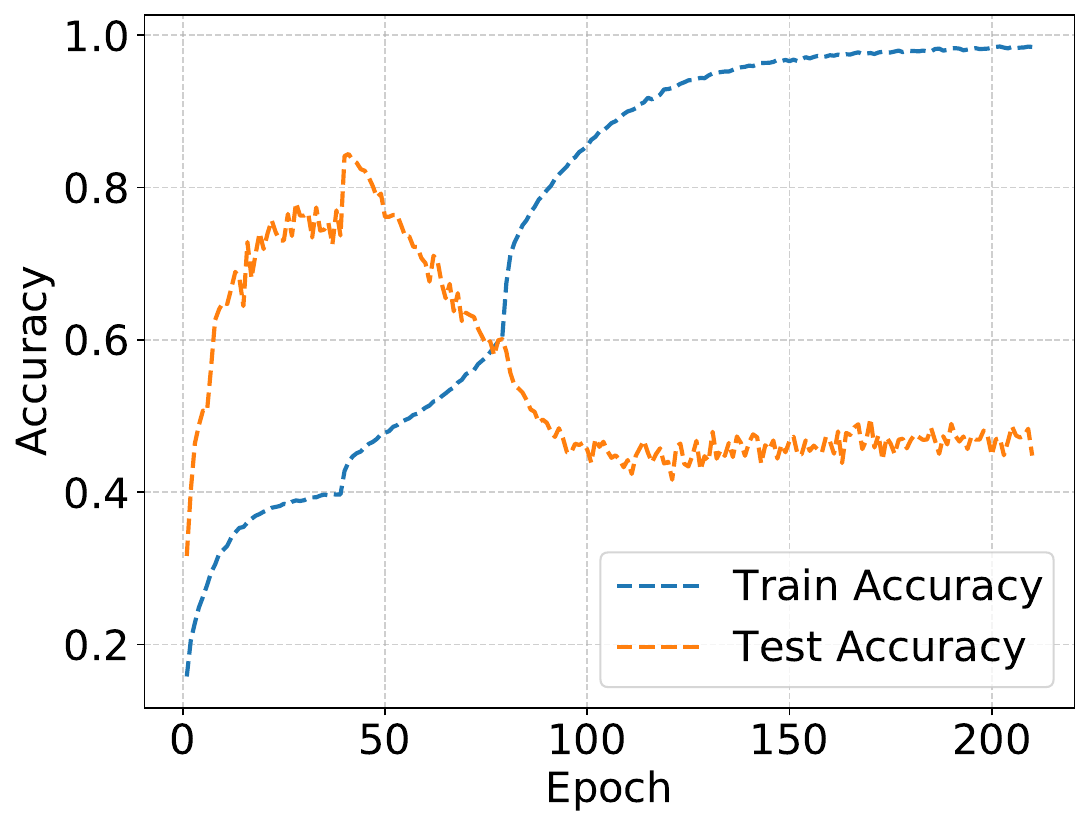}}
    \subfigure[CIFAR-100 Sym-60\%]{\includegraphics[width=0.49\textwidth]{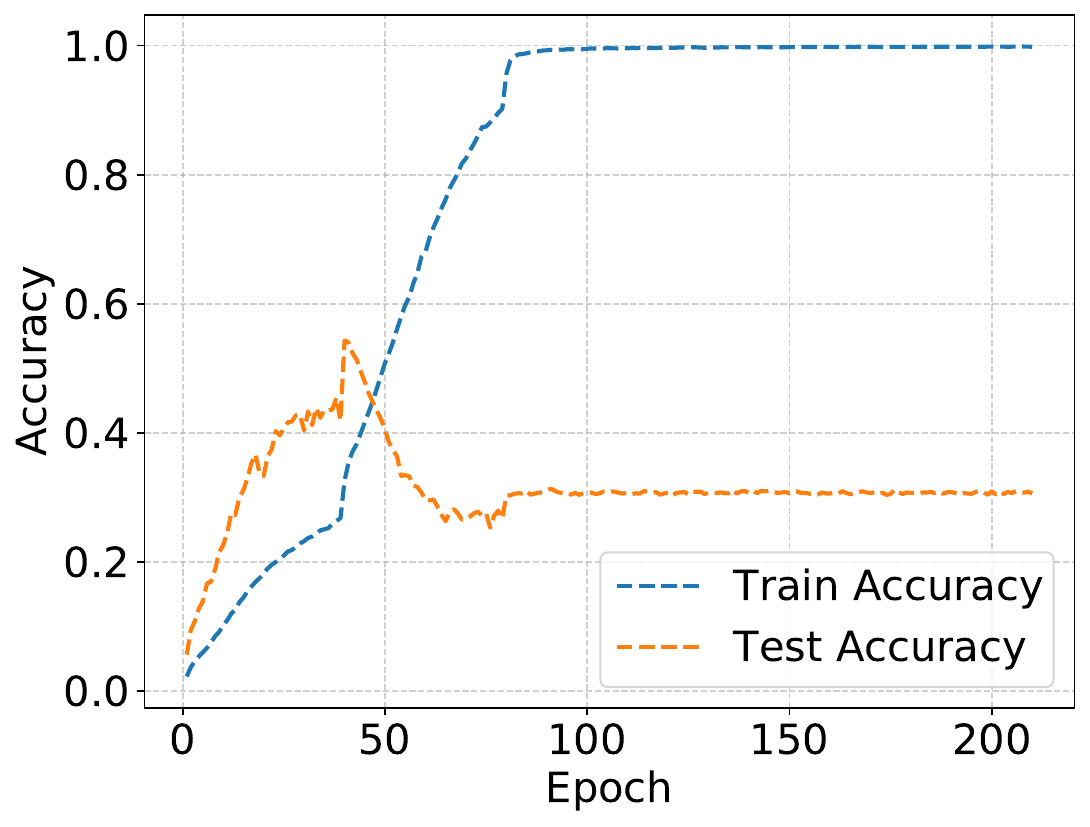}}
    
    \caption{Plots (a) and (b) illustrate the training and testing accuracies of DNNs employing cross-entropy loss on the CIFAR-10/100 datasets, each with 60\% symmetric label noise.}
    \label{fig_t&t}
\end{figure}

From Eq.(\ref{eq:g_ce}), it can be observed that the primary factors influencing the gradient are the model's predictions and the corresponding target values. For a given $\mathbf{x}_i$, let $c$ indicate the labeled class. During the early training phase, both clean and noisy samples exhibit significant gradients $\mathbf{p}_{i[c]} - \tilde{\mathbf{y}}_{i[c]}$, with the gradient differences between individual samples being relatively small. In this phase, the model focuses more on learning from clean samples, as they contribute more effectively to the learning process without significant interference from noisy samples.

However, as the model's performance improves, the gradient $\mathbf{p}_{i[c]} - \tilde{\mathbf{y}}_{i[c]}$ of clean samples approaches 0, while the gradient $\mathbf{p}_{i[c]} - \tilde{\mathbf{y}}_{i[c]}$ of noisy samples approaches $-1$. This shift causes the model to increasingly focus on noisy labels during the later training stages. As a result, noisy samples begin to dominate the learning process, resulting in misled model updates. Ultimately, the model memorizes the noisy labels entirely, further compromising its generalization capability.

\subsection{Problem Setting} \label{subseb32}
We investigate supervised classification under label noise, where a corrupted training dataset $\tilde{\mathcal{D}} = \{(\mathbf{x}_i, \tilde{\mathbf{y}}_i)\}_{i=1}^n$ consists of inputs $\mathbf{x}_i \in \mathbb{R}^d$ and noisy labels $\tilde{\mathbf{y}}_i \in (0, 1)^K$ for $K$ classes. The noisy labels $\tilde{\mathbf{y}}_i$ deviate from the true labels $\mathbf{y}_i$ with some probability, introducing errors that challenge the learning process.

We identify three key challenges in training DNNs with noisy data:
\begin{itemize}
\item DNNs initially learn clean samples effectively but overfit noisy labels in later stages, leading to performance degradation.
\item Noisy samples generate large gradients, which dominate the training process and skew the model.
\item Approaches that use manual thresholds to separate clean and noisy labels are error-prone and depend heavily on the dataset.
\end{itemize}

To overcome these challenges, we present the Adaptive Label Refinement (ALR) method that dynamically mitigates the impact of noisy labels while effectively utilizing clean samples, without requiring prior knowledge of noise characteristics.

\section{Methodology}\label{sec4}
This section introduces the Adaptive Label Refinement (ALR) method.  Section \ref{subsec41} introduces the overall framework of the ALR. We present the adaptive label refinement strategy in Section \ref{subsec42}. Finally, \ref{subsec43} provides a theoretical analysis.

\begin{figure}[htb]
    \centering
    \includegraphics[width=0.7\textwidth]{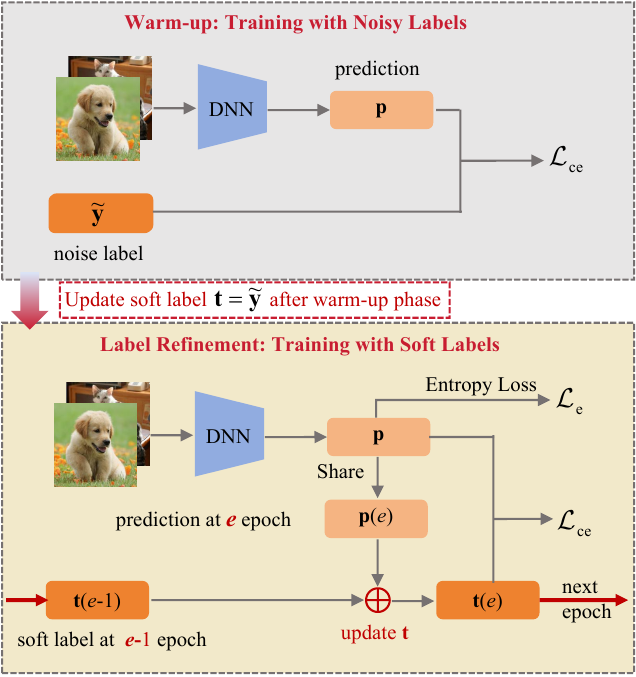}
    \caption{The architecture of the ALR training framework. In the warm-up phase, the model is trained using cross-entropy loss $\mathcal{L}_{ce}$ with the original noisy labels ($\mathbf{y}$). In the label refinement phase, the noisy labels are replaced with soft labels, and an entropy loss term $\mathcal{L}_{\text{e}}$ is incorporated to assist in training the model. At each epoch, the soft label $\mathbf{t}(e)$ is updated by blending the soft label from the previous epoch $\mathbf{t}(e-1)$ and the current prediction $\mathbf{p}(e)$, using a weighted sum.}
    \label{fig_alr}
\end{figure}

\subsection{Overall Framework of ALR}\label{subsec41}
Figure \ref{fig_alr} presents the key components and overall process of the ALR. As illustrated in the figure, the training consists of two phases in the ALR framework: the warm-up phase and the label refinement phase. As noted in Section \ref{subsec31}, an early learning phenomenon occurs when training on datasets with mislabeled samples. Previous work \citep{arpit2017closer} has demonstrated that the model's predictions for clean and mislabeled samples are relatively accurate in the early training phase. Leveraging this property, during the initial m training epochs, we use only the original labels and optimize with the standard cross-entropy loss. This stage, referred to as the warm-up phase, enables the model to generate relatively reliable predictions, which form the foundation for the subsequent adaptive label refinement.

In the label refinement phase, we progressively transform the original hard labels (one-hot labels) into more accurate soft labels during training, following the strategy outlined in Section \ref{subsec42}. The use of soft labels instead of hard labels mitigates the adverse effects of incorrect labels, allowing for iterative improvement in label quality.

Simultaneously, we introduce a minimum entropy regularization term during training iterations to refine the soft labels of high-confidence samples, gradually bringing them closer to one-hot encoded vectors. This approach facilitates adaptive learning of clean labels while reducing the risk of memorizing noisy labels, ultimately enhancing model performance.

\subsection{Adaptive Label Refinement}\label{subsec42}
In order to tackle the negative impact of label noise, we decouple the tasks of thoroughly learning clean labels and avoiding fitting incorrect labels. As discussed in Section \ref{subsec31}, the gradient in Eq.(\ref{eq:g_ce}) plays a key role in guiding the model's learning of both clean and noisy samples. With the progression of training, the model tends to overfit noisy labels because the gradient $\mathbf{p}_{i[u]} - \tilde{\mathbf{y}}_{i[u]}$ of noisy samples remains large, dominating the training process. To prevent noisy samples from dominating the training process, maintaining the gradient of noisy samples smaller than that of clean samples throughout training is the key challenge. 

We propose an adaptive label refinement strategy to tackle this key challenge. To reduce the model's sensitivity to incorrect labels, we progressively transform the original one-hot labels into probability distributions based on the model's predictions. By gradually shifting the learning objective from hard labels to predicted probability distributions, the magnitude of the labeled class $\tilde{\mathbf{y}}_{i[u]}$ is reduced, thereby decreasing the gradient $\mathbf{p}_{i[u]} - \tilde{\mathbf{y}}_{i[u]}$. For noisy samples, this reduction in the gradient mitigates the model’s sensitivity to incorrect labels.

Inspired by SELC \cite{lu2022selc}, We adopt the temporal ensembling to generate soft labels as the model's new learning targets by blending original one-hot labels with its historical predictions. The target $\mathbf{t}(e)$ at the epoch $e$ during the label improvement phase is defined as follows:

\begin{equation}
    t(e) = 
        \begin{cases} 
        \tilde{\mathbf{y}} & \text{if } e <= m \\ 
        \alpha \mathbf{t}(e-1) + (1 - \alpha)\mathbf{p}(e), & \text{if } e > m
        \end{cases}
        \label{ep:label_t}
\end{equation}
where $\tilde{\mathbf{y}}_i$ is the original noisy label of $i^{th}$ sample, and $\mathbf{p}_{i}(e)$ is the model's prediction for the $i^{th}$ sample at epoch $e$. The warming phase ends at epoch $m$, where $e \leq m$ indicates the warming phase and $e > m$ indicates the adaptive label refinement phase. $\alpha$ is the momentum, with $0 \leq \alpha < 1$.

Substituting Eq. (\ref{ep:label_t}) into Eq. (\ref{ep:label_t}) gives the loss function as:

\begin{equation}
\mathcal{L}_{ce} = -\frac{1}{n} \sum_{i=1}^{n} \sum_{k=1}^{K} {\mathbf{t}_{i[k]}}(e) \log {\mathbf{p}_{i[k]}}(e), 
\label{ep:noise_ce}
\end{equation}
where $\mathbf{t}_{i[k]}(e)$ denotes the value of class $k$ in the target $t_i$ for the $i^{th}$ sample at epoch $e$, and $\mathbf{p}_{i[k]}(e)$ is the corresponding value in the model's prediction $\mathbf{p}_i$. When $e \leq m$, the learning targets are the original hard labels. For $e > m$,  the model's learning targets are gradually updated from hard to soft labels.

Although soft labels mitigate the model's fitting to incorrect labels, they also hinder learning from clean labels, reducing the utilization of clean samples. Traditional methods rely on sample selection strategies to distinguish clean samples from noisy ones, aiming to maximize the use of clean data. However, such methods often require complex models and are highly dataset-dependent. To overcome these limitations, we implicitly filter clean samples by introducing an entropy loss regularization term, defined as follows:
\begin{equation}
\mathcal{L}_{\text{e}} = -\sum_{i=1}^{n} \sum_{k=1}^{K} {\mathbf{p}_{i[k]}}(e) \log {\mathbf{p}_{i[k]}}(e).
\end{equation}

The overall loss function is defined as follows:

\begin{equation}
    \mathcal{L}_{ALR} = \mathcal{L}_{ce} + \lambda \mathcal{L}_{\text{e}}, 
    \label{ep:alr}
\end{equation}
where $\lambda$ is a hyperparameter that governs the relative importance of the entropy loss term in the overall loss function. 

During iterative training, the regularization term causes the soft labels of high-confidence samples to harden gradually, pushing their probability distribution closer to one-hot encoding. By leveraging clean samples more comprehensively, the model's predictive accuracy on noisy samples is further improved as well.

The regularization weight, $\lambda$, should be set to a relatively small value, such as $0.2$. This ensures that soft labels for high model confidence samples are hardened effectively, allowing for accurate and robust model updates. Simultaneously, it prevents the overcorrection of potential errors in the soft labels for low-confidence samples. By balancing these effects, the model learns effectively from clean samples, enhancing performance while avoiding overfitting to noisy labels. The Algorithm \ref{algo1} shows the detailed pseudocode for ALR.

\begin{algorithm}
\caption{Pseudocode for ALR.}\label{algo1}
\textbf{Input:} DNNs $\mathcal{N}_{\theta}$, noisy training data $\tilde{\mathcal{D}} = \{(\mathbf{x}_i, \tilde{\mathbf{y}}_i)\}_{i=1}^n$, total epoch $E$, warm-up phase epoch $m$, hyperparameter $\alpha$ and $\lambda$\\
\textbf{Output:} Optimized DNNs $\mathcal{N}_{\theta^*}$
\begin{algorithmic}[1]
 \State Let $\mathbf{t}(0) = \tilde{\mathbf{y}}$. \Comment{initialize soft labels $\mathbf{t}$}
\While{epoch $e < E$}
    \If{epoch $e < m$} \Comment{warm-up phase}
        \For{$b=1$ to $B$} \Comment{$b \in (1,...,B)$, $B$ is number of mini-batch}
            \For{$i$ in mini-batch $b$} 
                \State $\mathbf{t}_i(e) = \tilde{\mathbf{y}_i}$
                \State $\mathbf{p}_i(e) = \mathcal{S}(\mathcal{N}_{\theta}(\mathbf{x}_i))$ \Comment{prediction of $\mathbf{x}_i$}
                \State $\mathcal{L}_{ce} = -\frac{1}{n} \sum_{i=1}^{n} {\mathbf{t}_{i}}(e) \log {\mathbf{p}_{i}(e)}$
                \State Update $\mathcal{N}_{\theta}$ using SGD. \Comment{update DNNs parameters}
            \EndFor
        \EndFor
    \Else \Comment{label refinement phase}
        \For{$b=1$ to $B$} \Comment{$b \in (1,...,B)$, $B$ is number of mini-batch}
            \For{$i$ in mini-batch $b$} 
                \State $\mathbf{p}_i(e) = \mathcal{S}(\mathcal{N}_{\theta}(\mathbf{x}_i))$ \Comment{prediction of $\mathbf{x}_i$}
                \State $\mathbf{t}_i(e) = \alpha \mathbf{t}_i(e-1) + (1 - \alpha)\mathbf{p}_i(e)$ \Comment{update soft label $\mathbf{t}_i(e)$} 
                \State $\mathcal{L}_{ALR} = -\frac{1}{n} \sum_{i=1}^{n} {\mathbf{t}_{i}}(e) \log {\mathbf{p}_{i}(e)}$ \Comment{cross-entropy loss}
                                     \State\hspace{4em}$-\frac{\lambda}{n} \sum_{i=1}^{n} \sum_{k=1}^{K} {\mathbf{p}_{i[k]}}(e) \log {\mathbf{p}_{i[k]}}(e)$ \Comment{entropy loss}
                
                \State Update $\mathcal{N}_{\theta}$ using SGD. \Comment{update DNNs parameters}
            \EndFor
        \EndFor
    \EndIf
\EndWhile
\end{algorithmic}
\end{algorithm}

\subsection{Theoretical Analysis of ALR}\label{subsec43}

In this section, we analyze how ALR effectively promotes learning clean samples and avoiding overfitting noisy labels from a gradient perspective. During the label refinement phase, the learning target is the soft label $\mathbf{t}$, where the probability $\mathbf{t}_{[k]}$ of class $k$ satisfies $\mathbf{t}_{[k]} \in [0,1]$, and the sum of all class probabilities equals 1, i.e., $\sum_{k=1}^K  \mathbf{t}_{[k]} = 1$. Similarly, The model's prediction output is $\mathbf{p}$, with the predicted probability $\mathbf{p}_{[k]}$ of class $k$ also satisfies $\mathbf{p}_{[k]} \in [0,1]$, and $\sum_{k=1}^K \mathbf{p}_{[k]} = 1$. For simplicity, assume that the class with the highest predicted probability for a sample is $u$. The following theorem is presented to support the analysis.

\begin{theorem}
Given the adaptive label refinement loss $\ell_{ALR}$ in Eq. (\ref{ep:alr}), the sample-wise loss can be rewritten as: $\ell_{ALR} = -\sum_{k=1}^{K} {\mathbf{t}_{[k]}} \log {\mathbf{p}_{[k]}} - \sum_{k=1}^{K} {\mathbf{p}_{[k]}} \log {\mathbf{p}_{[k]}}$. For convenience, let $H \left( \mathbf{p} \right) = -\sum_{k=1}^{K} {\mathbf{p}_{[k]}} \log {\mathbf{p}_{[k]}}$ representing the entropy of the prediction $\mathbf{p}$. The gradient of $\ell_{ALR}$ with respect to the logits $\mathbf{z}_{[u]}$ is as follow: 
\begin{equation}
\frac{\partial \ell_{ALR}}{\partial \mathbf{z}_{[u]}} =  \mathbf{p}_{[u]} \left( 1 - \lambda \left( \log{\mathbf{p}_{[u]}} + H \left( \mathbf{p} \right) \right) \right) - \mathbf{t}_{[u]}, 
\label{eq:th_alr}
\end{equation}
where $\mathbf{z}_{[u]}$ is the $c^{th}$ logit value, $\lambda \in (0,1)$. Additionally, there exists $\varepsilon \in (0,1)$ such that
\begin{equation}
    \begin{cases} 
    \frac{\partial \ell_{ALR}}{\partial \mathbf{z}_{[u]}} < \mathbf{p}_{[u]} - \mathbf{t}_{[u]}, & \text{if } \mathbf{p}_{[u]} > \varepsilon, \\ 
    \frac{\partial \ell_{ALR}}{\partial \mathbf{z}_{[u]}} \geq \mathbf{p}_{[u]} -\mathbf{t}_{[u]}, & \text{if } \mathbf{p}_{[u]} \leq \varepsilon, 
    \end{cases} \label{eq:th_logit}
\end{equation}
where $\mathbf{p}_{[u]} > \varepsilon$ denotes a high-confidence prediction, while $\mathbf{p}_{[u]} \leq \varepsilon$ signifies a low-confidence prediction.  \label{th_alr}
\end{theorem}

\begin{proof}
In the label refinement phase, the sample-wise loss $\ell_{ALR}$ as follow:

\begin{equation}
\ell_{ALR} = -\sum_{k=1}^{K} {\mathbf{t}_{[k]}} \log {\mathbf{p}_{[k]}} - \sum_{k=1}^{K} {\mathbf{p}_{[k]}} \log {\mathbf{p}_{[k]}}.
\end{equation}

The gradient of the loss $\ell_{ALR}$ with respect to the logits $\mathbf{z}$  is calculated as follows:

\begin{equation}
\frac{\partial \ell_{ALR}}{\partial \mathbf{z}_{[u]}} = - \sum_{k=1}^{K} \frac{\mathbf{t}_{[k]}}{\mathbf{p}_{[k]}} \frac{\partial \mathbf{p}_{[k]}}{\partial \mathbf{z}_{[u]}} - \lambda \sum_{k=1}^{K} (1 + \log{\mathbf{p}_{[k]}}) \frac{\partial \mathbf{p}_{[k]}}{\partial \mathbf{z}_{[u]}}, 
\end{equation}
where $\mathbf{p} = softmax(\mathbf{z}) = \frac{e^{\mathbf{z}_{[k]}}}{\sum_{k=1}^{K} e^{\mathbf{z}_{[k]}}}$, we have
\begin{equation}
    \frac{\partial \mathbf{t}_{[k]}}{\partial \mathbf{z}_{[u]}} = \frac{\partial}{\partial \mathbf{z}_{[u]}} \left( \frac{e^{\mathbf{z}_{[k]}}}{\sum_{k=1}^K e^{\mathbf{z}_{[u]}}} \right) 
= \frac{\frac{\partial e^{\mathbf{z}_{[k]}}}{\partial \mathbf{z}_{[k]}} \left( \sum_{k=1}^K e^{\mathbf{z}_{[u]}} \right) - e^{\mathbf{z}_{[k]}} \frac{\partial \left( \sum_{k=1}^K e^{\mathbf{z}_{[u]}} \right)}{\partial {\mathbf{z}_{[u]}}}}{\left( \sum_{k=1}^K e^{\mathbf{z}_{[u]}} \right)^2}.
\end{equation}

For $k=u$:
\begin{align}
    \frac{\partial \mathbf{p}_{[k]}}{\partial \mathbf{z}_{[u]}} 
& = \frac{e^{\mathbf{z}_{[k]}} \left( \sum_{k=1}^K e^{\mathbf{z}_{[k]}} \right) - e^{\mathbf{z}_{[k]}} \cdot e^{\mathbf{z}_{[k]}}}{\left( \sum_{k=1}^K e^{\mathbf{z}_{[k]}} \right)^2} \nonumber  \\&
= \frac{e^{\mathbf{z}_{[k]}}}{\sum_{k=1}^K e^{\mathbf{z}_{[k]}}} - \left( \frac{e^{\mathbf{z}_{[k]}}}{\sum_{k=1}^K e^{\mathbf{z}_{[k]}}} \right)^2
= \mathbf{p}_{[k]} - \mathbf{p}_{[k]}^2.
\end{align}

For $k \neq u$:
\begin{equation}
    \frac{\partial \mathbf{p}_{[k]}}{\partial \mathbf{z}_{[u]}} = \frac{0 \cdot \left( \sum_{u=1}^K e^{\mathbf{z}_{[u]}} \right) - e^{\mathbf{z}_{[k]}} \cdot e^{\mathbf{z}_{[u]}}}{\left( \sum_{k=1}^K e^{\mathbf{z}_{[u]}} \right)^2} 
= - \frac{e^{\mathbf{z}_{[u]}} e^{\mathbf{z}_{[u]}}}{\left( \sum_{k=1}^K e^{\mathbf{z}_{[u]}} \right)^2} 
= - \mathbf{p}_{[k]} \mathbf{p}_{[u]}.
\end{equation}

By combining Eq. (3) and (4), we can derive the following:
\begin{align}
\frac{\partial \ell_{ALR}}{\partial \mathbf{z}_{[u]}} 
&= - \sum_{k=1}^{K} \frac{\mathbf{t}_{[k]}}{\mathbf{p}_{[k]}} \frac{\partial \mathbf{p}_{[k]}}{\partial \mathbf{z}_{[u]}} - \lambda \sum_{k=1}^{K} \left(1 + \log{\mathbf{p}_{[k]}}\right) \frac{\partial \mathbf{p}_{[k]}}{\partial \mathbf{z}_{[u]}} \nonumber\\
&= - \left( \frac{\mathbf{t}_{[u]}}{\mathbf{p}_{[u]}} + \lambda \left( \log{\mathbf{p}_{[u]}} + 1 \right) \right) \frac{\partial \mathbf{p}_{[u]}}{\partial \mathbf{z}_{[u]}} - \sum_{k \neq c} \left( \frac{\mathbf{t}_{[k]}}{\mathbf{p}_{[k]}} + \lambda \left( \log{\mathbf{p}_{[k]}} + 1 \right) \right) \frac{\partial \mathbf{p}_{[k]}}{\partial \mathbf{z}_{[k]}} \nonumber \\
&= - \left( \frac{\mathbf{t}_{[u]}}{\mathbf{p}_{[u]}} + \lambda \left( \log{\mathbf{p}_{[u]}} + 1 \right) \right) \left( \mathbf{p}_{[u]}-\mathbf{p}_{[u]}^2 \right)  \nonumber \\  &\hspace{0.5cm} - \sum_{k \neq c} \left( \frac{\mathbf{t}_{[k]}}{\mathbf{p}_{[k]}} + \lambda \left( \log{\mathbf{p}_{[k]}} + 1 \right) \right) \left( -\mathbf{p}_{[k]}\mathbf{p}_{[u]} \right) \nonumber \\
&= \mathbf{p}_{[u]} - \mathbf{t}_{[u]} - \lambda \mathbf{p}_{[u]} \left( \log{\mathbf{p}_{[u]}} + H(\mathbf{p}) \right) \nonumber \\
&= \mathbf{p}_{[u]} \left( 1 - \lambda \left( \log{\mathbf{p}_{[u]}} + H \left( \mathbf{p} \right) \right) \right) - \mathbf{t}_{[u]},  \label{eq_z_alr}
\end{align}
where $H \left( \mathbf{p} \right) =  -\sum_{k=1}^{K} {\mathbf{p}_{[k]}} \log {\mathbf{p}_{[k]}}$. \\

The function $ f(\mathbf{p}_{[u]}) $ is defined as:

\begin{equation}
f(\mathbf{p}_{[u]}) = \log \mathbf{p}_{[u]} + H(\mathbf{p}),
\end{equation}

The properties of \(f(\mathbf{p}_{[u]})\) are analyzed as follows:

\textbf{Continuity and Existence of a root.}
The function $f(\mathbf{p}_{[u]})$ is continuous on $(0, 1)$, and its behaviour at the boundaries of the interval is as follows:
   \begin{itemize}
       \item As $\mathbf{p}_{[u]} \to 0^+$, $\log \mathbf{p}_{[u]} \to -\infty$, so $f(\mathbf{p}_{[u]}) \to -\infty$;
       \item As $\mathbf{p}_{[u]} \to 1^-$, $\log \mathbf{p}_{[u]} \to 0$, and since $H(\mathbf{p}) > 0$, $f(\mathbf{p}_{[u]}) \to H(\mathbf{p}) > 0$.
   \end{itemize}

By the \textit{Intermediate Value Theorem}, since $ f(\mathbf{p}_{[u]}) $ is continuous on $ (0, 1) $ and tends to $ -\infty $ as $ \mathbf{p}_{[u]} \to 0^+ $ and $ f(\mathbf{p}_{[u]}) > 0 $ as $ \mathbf{p}_{[u]} \to 1^- $, there must exists an $ \varepsilon \in (0, 1) $ such that:

\begin{equation}
f(\varepsilon) = \log \varepsilon + H(\mathbf{p}) = 0.
\end{equation}

This shows that the function $ f(\mathbf{p}_{[u]}) $ has at least one root in the interval $(0, 1)$.

\textbf{Uniqueness of the root.}
To prove that $\varepsilon$ is the only root of $f(\mathbf{p}_{[u]})$ in $(0, 1)$, we calculate the derivative of $ f(\mathbf{p}_{[u]})$ for $\mathbf{p}_{[u]} \in (0,1)$:

\begin{align}
f'(\mathbf{p}_{[u]}) &= \frac{1}{\mathbf{p}_{[u]}} + \frac{\partial H(\mathbf{p})}  {\partial \mathbf{p}_{[u]}} \\
&= \frac{1}{\mathbf{p}_{[u]}} - \left( \log{\mathbf{p}_{[k]}} + 1 \right) > 0.
\end{align}

Thus, $ f(\mathbf{p}_{[u]}) $ is strictly monotonic on $ (0, 1) $, ensuring the existence of at most one root in this interval.

\textbf{Sign Analysis.}
We analyze the sign of $f(\mathbf{p}_{[u]})$ as:

\begin{equation}
f(\mathbf{p}_{[u]}) = 
    \begin{cases} 
    \log \mathbf{p}_{[u]} + H(\mathbf{p}) > 0, & \text{if } \mathbf{p}_{[u]} > \varepsilon, \\
    \log \mathbf{p}_{[u]} + H(\mathbf{p}) \leq 0, & \text{if } \mathbf{p}_{[u]} \leq \varepsilon.
    \end{cases} \label{eq_sign}
\end{equation}

In conclusion, there exists a unique $\varepsilon \in (0, 1)$ such that $f(\mathbf{p}_{[u]}) = 0$. Moreover, $f(\mathbf{p}_{[u]}) > 0$ for $\mathbf{p}_{[u]} > \varepsilon$, and $f(\mathbf{p}_{[u]}) \leq 0$ for $\mathbf{p}_{[u]} \leq \varepsilon$.

By substituting Eq.(\ref{eq_sign}) into Eq.(\ref{eq_z_alr}), we can prove Eq(\ref{eq:th_logit}) for $\lambda \in (0, 1)$.

\end{proof}

During the label refinement phase,  the learning target $\mathbf{t}(e)$ is updated as $\alpha \mathbf{t}(e-1) + (1 - \alpha)\mathbf{p}(e)$. As $\mathbf{t}(e)$ approaches $\mathbf{p}(e)$, the gradient $\frac{\partial \ell_{ALR}}{\partial \mathbf{z}_{[u]}}$ gradually decreases, thus reducing the model's tendency to overfit incorrect labels. Furthermore, Theorem \ref{th_alr} guarantees that the model successfully leverages clean samples for learning. High-confidence samples are typically clean, as larger values of $\mathbf{p}_{[u]}$ suggest a higher likelihood of the label being correct.

For clean samples ($\mathbf{p}_{[u]}>\varepsilon$), the gradient of ALR $\frac{\partial \ell_{ALR}}{\partial \mathbf{z}_{[u]}}$ is smaller than the gradient of cross entropy loss $\frac{\partial \ell_{ce}}{\partial \mathbf{z}_{[u]}} = \mathbf{p}_{[u]} - \mathbf{t}_{[u]}$.  Since the gradient of the model update is negative, $\frac{\partial \ell_{ALR}}{\partial \mathbf{z}_{[u]}} < \frac{\partial \ell_{ce}}{\partial \mathbf{z}_{[u]}}$ enables the model to prioritize fitting clean samples more effectively. 

As the model continues learning from clean samples during the label refinement phase, its overall performance improves, leading to enhanced predictive accuracy even for noisy samples. Consequently, ALR enhances classification performance by avoiding overfitting to noisy labels while thoroughly learning from clean samples.

\section{Experiment}\label{sec5}
This section evaluates the proposed ALR method across five diverse datasets. Section \ref{subsec51} introduces the comparative methods. Section \ref{subsec52} examines classification performance on benchmark datasets with artificial noise. Section \ref{subsec53} evaluates performance on three real-world datasets. Section \ref{subsec54} provides empirical analysis, while Section \ref{subsec55} evaluates the impact of different components through ablation experiments. Finally, Section \ref{subsec56} explores hyperparameter sensitivity.

\subsection{Comparative Methods}\label{subsec51}
To assess the effectiveness of ALR, we perform experiments comparing it against four types of methods, including Robust Loss Functions, Regularization, Sample Selection and Label Refinement. Below is a summary of these comparative approaches:

\textbf{Robust Loss Functions.} (1) GCE \citep{zhang2018generalized} combines cross-entropy loss and mean squared error, achieving a balance between the two by controlling hyperparameters. (2) SCE \citep{wang2019symmetric} enhances robustness to noisy labels by adding a reverse cross-entropy loss to the standard cross-entropy loss. 

\textbf{Regularization.} (1) ELR \citep{liu2020early} introduces a regularization term that utilizes the model's historical integrated predictions as the learning target. By leveraging early learning, it reduces the tendency to memorize noisy labels. (2) CCR \citep{cheng2022class} proposes forward-backward cycle consistency regularization, effectively reducing the estimation error of the transfer matrix. (3) NESTED \citep{chen2021boosting} introduces NESTED Dropout to regularize the neural network to resist label noise. 

\textbf{Sample Selection.} (1) NLNL \citep{kim2019nlnl}, which proposes a negative learning approach utilizing complementary labels and subsequently combines negative learning with positive learning (using labeled data) to improve the identification and filtering of noisy samples. (2) Co-teaching \citep{han2018co} employs two networks, where each network selects samples with low loss values to guide the training of the other. (3) CoDis \citep{xia2023combating}, which also utilizes a dual network approach, training by selecting samples that exhibit high variability in predictions between the two networks. (4) TS$^3$-net \citep{jiang2022sparse} introduces a novel two-stream sample selection network comprising a sparse and dense sub-network. It also designs a multi-phase learning framework, enabling the sparse sub-network to achieve strong classification performance despite label noise. 

\textbf{Label refurbishment.} (1) Bootstrap \citep{2015TRAINING} reduces the model’s tendency to overfit noisy labels by combining the original label with the current model prediction to form the training target.. (2) SELC \citep{lu2022selc}, which leverages early learning phenomena to develop a sufficiently good model, subsequently using temporal integration to update the labels. (3) PENCIL \citep{yi2019probabilistic}, which proposes a novel training framework that updates the label estimates as probability distributions while concurrently updating the network parameters. (4) CWD \citep{gong2022class}, which develops an enhanced centroid estimator that utilizes the local information of each class, combating label noise by correcting the loss function for incorrect labels. (5) PLC \citep{prog_noise_iclr2021} proposes a progressive label correction algorithm that selectively corrects labels by establishing a continuously increasing threshold.

\subsection{Classification under Artificial Noise}\label{subsec52}

\textbf{Datasets and Artificial Noise.} 
We evaluate the classification performance of ALR under varying degrees of artificial label noise by artificially corrupting a portion of the labels in two benchmark datasets, CIFAR-10 and CIFAR-100 \citep{krizhevsky2009learning}. These datasets are widely used to assess methods designed to handle label noise. Detailed statistics for two datasets are presented in Table \ref{tab_ciafr}.

Since these datasets are originally clean, we generate two types of artificial label noise—symmetric and asymmetric—by using the methodology outlined in \citep{patrini2017making}. Symmetric noise is introduced by randomly selecting training samples according to the noise rate and uniformly replacing their labels with all possible classes. Asymmetric noise is generated by selecting samples from each class according to the noise rate and flipping their labels to the next class, iterating through all classes in this manner. For asymmetric noise in CIFAR-10, labels are flipped according to the mappings outlined in \citep{patrini2017making}: Truck \(\rightarrow\) Automobile, Bird \(\rightarrow\) Airplane, Deer \(\rightarrow\) Horse, and Cat \(\leftrightarrow\) Dog, following the specified noise rate. These artificial noises simulate real-world labeling errors and provide a robust evaluation framework for noise-handling methods.

\begin{table}[htbp]
    \caption{Description of two CIFAR datasets used in experiments with artificial label noise. }\label{tab_ciafr}%
    \begin{tabular*}{\textwidth}{@{}@{\extracolsep{\fill}}lcccc@{}}
    \toprule
    Dataset & \#Train & \#Test & \#Classes & Image Size \\
    \midrule
    CIFAR-10 & 50K & 10K & 10 & 32 $\times$ 32 \\
    CIFAR-100 & 50K & 10K & 100 & 32 $\times$ 32 \\
    \bottomrule‌
    \end{tabular*}
\end{table}

\textbf{Experimental Details.} 
All experiments maintained consistent data preprocessing and network architecture with previous studies to ensure fair comparisons with existing work. For data preprocessing, we employ a simple data augmentation strategy, which includes random cropping with 4 pixels of padding on each side and horizontal flipping, along with normalization. We employ ResNet-34 \citep{he2016deep} as the backbone network for both datasets.

The model is trained with a batch size of 128, using SGD as the optimizer, a momentum of 0.9, and a weight decay of 0.001. Training is conducted for 200 epochs on both datasets. Additionally, the initial learning rate was set to 0.02, reducing by a factor of 10 after 40 and 80 epochs. 

We maintain consistent hyperparameter settings under varying noise conditions in CIFAR-10 and CIFAR-100: the temporal integration parameter was set to 0.9, the entropy loss weight to 0.2, and the initial warm-up phase to 30 epochs. Many methods assume prior knowledge of noise characteristics. When the dataset and noise conditions change, further adjustments to the hyperparameters are often necessary to achieve satisfactory results. In contrast, we use the same hyperparameters to highlight the method’s robustness and effectiveness across various scenarios.

\begin{table}[htbp]
    \caption{The test accuracy (\%) on CIFAR-10 with symmetric and asymmetric label noise. The best results are shown in bold, while the second-best are underlined.}\label{tab_cifar10}
    \resizebox{\textwidth}{!}{
    \begin{tabular}{lcccccc}
    \toprule
    & \multicolumn{6}{c}{\textbf{CIFAR-10}} \\
    \midrule
    \multirow{2}{*}{\textbf{Method}} & \multicolumn{4}{c}{\textbf{Symm}} & \multicolumn{2}{c}{\textbf{Asymm}} \\
    \cmidrule(lr){2-5} \cmidrule(l){6-7}
    & 20\% & 40\% & 60\% & 80\% & 20\% & 40\% \\
    \midrule
    CE & 86.98$\pm$0.12 & 81.88$\pm$0.29 & 74.14$\pm$0.56 & 53.82$\pm$1.04 & 88.59$\pm$0.34 & 80.11$\pm$1.44 \\
    GCE & 89.83$\pm$0.20 & 87.13$\pm$0.22 & 82.54$\pm$0.23 & 64.07$\pm$1.38 & 89.33$\pm$0.17 & 76.74$\pm$0.61 \\
    Bootstrap & 86.23$\pm$0.23 & 82.23$\pm$0.37 & 75.12$\pm$0.56 & 54.12$\pm$1.32 & 88.26$\pm$0.24 & 81.21$\pm$1.47 \\
    PENCIL & - & - & - & - & 92.43 & 91.01 \\
    NLNL & \textbf{94.23} & \textbf{92.43} & \underline{88.32} & - & 93.35 & 89.86 \\
    SCE & 89.83$\pm$0.32 & 87.13$\pm$0.26 & 82.81$\pm$0.61 & 68.12$\pm$0.81 & 90.44$\pm$0.27 & 82.51$\pm$0.45 \\
    CCR & 90.44$\pm$0.19 & 87.30$\pm$0.25 & 81.01$\pm$0.25 & - & 90.55$\pm$0.03 & 87.29$\pm$0.05 \\
    ELR & 91.16$\pm$0.08 & 89.15$\pm$0.17 & 86.12$\pm$0.49 & 73.86$\pm$0.61 & \underline{93.52$\pm$0.23} & 90.12$\pm$0.47 \\
    SELC & 93.09$\pm$0.02 & 91.18$\pm$0.06 & 87.25$\pm$0.09 & \underline{74.13$\pm$0.14} & - & \underline{91.05$\pm$0.11} \\
    \cmidrule{1-7}
    ALR (our) & \underline{93.65$\pm$0.09} & \underline{91.89$\pm$0.20} & \textbf{89.00$\pm$0.23} & \textbf{78.47$\pm$0.6} & \textbf{94.40$\pm$0.14} & \textbf{92.84$\pm$0.07} \\
    \bottomrule
    \end{tabular}
    }
\end{table}

\setlength{\tabcolsep}{4pt} 
\begin{table}[htbp]
    \caption{Test accuracy (\%) on CIFAR-100 with artificially injected symmetric and asymmetric label noise in the training set.}\label{tab_cifar100}
    \resizebox{\textwidth}{!}{
    \begin{tabular}{lcccccc}
    \toprule
    & \multicolumn{6}{c}{\textbf{CIFAR-100}} \\
    \midrule
    \multirow{2}{*}{\textbf{Method}} & \multicolumn{4}{c}{\textbf{Symm}} & \multicolumn{2}{c}{\textbf{Asymm}} \\
    \cmidrule(lr){2-5} \cmidrule(l){6-7}
    & 20\% & 40\% & 60\% & 80\% & 20\% & 40\% \\
    \midrule
    CE & 58.72$\pm$0.26 & 48.20$\pm$0.65 & 37.41$\pm$0.94 & 18.10$\pm$0.82 & 59.20$\pm$0.18 & 42.74$\pm$0.61 \\
    GCE & 66.81$\pm$0.42 & 61.77$\pm$0.24 & 53.16$\pm$0.78 & 29.16$\pm$0.74 & 66.59$\pm$0.22 & 47.22$\pm$1.15 \\
    Bootstrap & 58.27$\pm$0.21 & 47.66$\pm$0.55 & 34.68$\pm$1.10 & 21.64$\pm$0.97 & 62.14$\pm$0.32 & 45.12$\pm$0.57 \\
    PENCIL & - & \underline{69.12$\pm$0.62} & 57.70$\pm$3.86 & fail & \underline{74.70$\pm$0.56} & 63.61$\pm$0.23 \\
    NLNL & 71.52 & 66.39 & 56.51 & - & 63.12 & 45.70 \\
    SCE & 70.38$\pm$0.13 & 62.27$\pm$0.22 & 54.82$\pm$0.57 & 25.91$\pm$0.44 & 72.56$\pm$0.22 & 69.32$\pm$0.87 \\
    CCR & 67.74$\pm$0.17 & 61.71$\pm$0.20 & 49.30$\pm$0.82 & - & 68.34$\pm$0.24 & 62.64$\pm$0.49 \\
    ELR & \underline{74.21$\pm$0.22} & 68.28$\pm$0.31 & 59.28$\pm$0.67 & 29.78$\pm$0.56 & 74.03$\pm$0.31 & \underline{73.26$\pm$0.64} \\
    SELC & 73.63$\pm$0.07 & 68.46$\pm$0.10 & \underline{59.41$\pm$0.06} & \textbf{32.63$\pm$0.06} & - & 70.82$\pm$0.09 \\
    \cmidrule{1-7}
    ALR (our) & \textbf{74.50$\pm$0.19} & \textbf{69.80$\pm$0.38} & \textbf{62.56$\pm$0.36} & \underline{31.96$\pm$0.54} & \textbf{77.06$\pm$0.13} & \textbf{75.02$\pm$0.18} \\
    \bottomrule
    \end{tabular}
    }
\end{table}

\noindent\textbf{Experimental Results}
We evaluate the performance of ALR on the CIFAR-10 and CIFAR-100 datasets, which are corrupted with symmetric and asymmetric label noise at different noise levels. The results are shown in Tables \ref{tab_cifar10} and \ref{tab_cifar100}. The reported results represent the mean accuracy and standard deviation. Results for other methods are taken directly from their respective original papers.

From these tables, we observe that the ALR method outperforms all compared methods under most noise on both datasets, demonstrating the effectiveness and robustness of our approach. Under asymmetric label noise, the small loss values of mislabeled samples make it difficult to distinguish clean samples from noisy ones based on the loss criterion. Many methods perform poorly under asymmetric noise because they rely on small loss criteria or confidence. Compared to baseline methods, the ALR method significantly improves performance under asymmetric label noise. Because ALR does not simply use confidence in a mechanical way, it indirectly enhances the learning of clean samples by influencing the gradient. Overall, the ALR effectively mitigates the detrimental effects of label noise and outperforms other methods.

\subsection{Classification under Real-World Label Noise}\label{subsec53}

\noindent\textbf{Dataset Introduction.} 
We evaluate the performance of ALR on three real-world datasets: ANIMAL-10 \citep{song2019selfie}, Clothing1M \citep{xiao2015learning}, and WebVision \citep{li2017webvision}. These image datasets are collected from the Internet and contain incorrect or ambiguous annotations, providing a robust framework for assessing ALR's effectiveness in addressing label noise challenges in practical scenarios.

Among these, ANIMAL-10 consists of images of animals across ten categories, all annotated by humans. It includes five pairs of easily confused animals, resulting in an annotation error rate of approximately 8\%. Clothing1M consists of 1 million images from 14 categories. This dataset features real noisy labels, with an estimated annotation error rate of 38.5\%. 

WebVision is a dataset containing 2.4 million images, which are annotated with label information for 1,000 classes derived from ImageNet. Consistent with prior work \citep{chen2019understanding}, we use the first 50 classes (approximately 66K samples) from the Google Images subset of WebVision, along with the ILSVRC-2012 validation set for training and testing. The estimated probability of labeling errors in this dataset is approximately 20\%. Additionally, the WebVision dataset exhibits a highly imbalanced data distribution, further challenging model performance. Table \ref{tab_real} summarizes the details of the three real-world datasets.

\begin{table}[htbp]
    \caption{Description of the three real-world datasets with label noise}\label{tab_real}%
    
    \begin{tabular*}{\textwidth}{@{}@{\extracolsep{\fill}}lcccccc@{}}
    \toprule
    Data set & \#Train & \#Val & \#Test & \#Classes & Image Size & Noise rate(\%)\\
    \midrule
    ANIMAL-10 & 50K & - & 5K & 10 & 64 $\times$ 64 & 8.0 \\
    Clothing1M & 1M & 14K & 10K & 14 & 224 $\times$ 224 & 38.5 \\
    Webvision1.0 & 66K & - & 2.5K & 50 & 256 $\times$ 256 & 20.0 \\
    \bottomrule
    \end{tabular*}
\end{table}

\noindent\textbf{Experimental Details.}
During data preprocessing, we apply basic augmentation techniques to the training sets of the three datasets, including random cropping, horizontal flipping, and normalization. Specifically, the ANIMAL-10 dataset does not involve random cropping. For the Clothing1M dataset, images are initially resized to 256 $\times$ 256 pixels before random cropping to 224 $\times$ 224 pixels. Meanwhile, images in the WebVision dataset are randomly cropped to dimensions of 227 $\times$ 227 pixels. 

The training for all three datasets is conducted using SGD, with a momentum of 0.9 and a weight decay of 0.001 (0.0005 for the WebVision dataset). For the ANIMAL-10 dataset, we use batch-normalized VGG19 as the backbone network with a batch size of 64 \citep{song2019selfie}. The initial learning rate is 0.02, which is divided by 10 after 50 and 75 epochs (100 epochs in total). In the case of the Clothing1M dataset, we use ResNet-50 \citep{he2016deep} pre-trained on the ImageNet dataset as the backbone network with a batch size of 64. To maintain the data distribution with noisy labels balanced, 2000 batches are randomly sampled from the training set in each epoch. The initial learning rate is 0.001, which decays by a factor of 10 after 10 and 20 epochs (30 epochs in total). For the WebVision dataset, we use Inception-ResNetV2 \citep{szegedy2016rethinking} as the backbone network with a batch size of 32. The initial learning rate is 0.02, which is multiplied by 0.1 after 40 and 80 epochs (100 epochs in total). 

\setlength{\tabcolsep}{6.5pt} 
\begin{table}[htbp]
    \caption{Classification accuracy (\%) results on ANIMAL-10. All methods use VGG19.}\label{tab_an10}
    \resizebox{\textwidth}{!}{
    \begin{tabular}{ccccccc}
    \toprule
     Cross Entropy & Nested & PLC & CWD & SELC & TS$^3$-Net & ALR(our) \\
    \midrule
    79.40$\pm$0.14 & 81.30$\pm$0.60 & 83.40$\pm$0.43 & 82.52 & \underline{83.73$\pm$0.06} & 81.36 & \textbf{85.79$\pm$0.05} \\
    \bottomrule
    \end{tabular}
    }
\end{table}

\setlength{\tabcolsep}{5pt} 
\begin{table}[htbp]
    \caption{Classification accuracy (\%) results on Clothing1M. All methods use a ResNet50 pre-trained on the ImageNet dataset.}\label{tab_1M}
    \resizebox{\textwidth}{!}{
    \begin{tabular}{cccccccccc}
    \toprule
    Cross Entropy & SCE & FINE & Nested & PLC & ELR & SELC & CoDis & TS$^3$-Net & ALR(our) \\
    \midrule
    69.21 & 71.02 & 72.91 & 73.10 & \underline{74.02} & 72.87 & 74.01 & 71.60 & 72.2 & \textbf{74.22} \\
    \bottomrule
    \end{tabular}
    }
\end{table}

\setlength{\tabcolsep}{4pt} 
\begin{table}[htbp]
    \caption{Classification accuracy (\%) results on (mini) WebVision. All methods use InceptionResNetV2.}\label{tab_web}
    \resizebox{\textwidth}{!}{
    \begin{tabular}{lcccccccc}
        \toprule
        & Method & Co-teaching & Iterative-CV & SELC & CoDis & TS$^3$-Net & ELR & ALR(our) \\
        \midrule
        WebVision & top1 & 63.58 & 65.24 & 74.38 & 63.80 & 73.48 & \textbf{76.26} & \underline{74.72} \\
                  & top5 & 85.20 & 85.34 & 90.66 & 85.54 & \underline{90.92} & \textbf{91.26} & 90.84 \\
        \midrule
        ILSVRC12 & top1 & 61.48 & 61.60 & \underline{70.85} & 62.26 & 70.27 & 68.71 & \textbf{71.56} \\
                  & top5 & 84.70 & 84.98 & \underline{90.74} & 85.39 & 90.46 & 87.84 & \textbf{91.00} \\
        \bottomrule
    \end{tabular}
    }
\end{table}

\noindent\textbf{Experimental Results.}
We evaluate the performance of ALR on three real-world datasets, and the results are presented in Tables \ref{tab_an10}, \ref{tab_1M}, and \ref{tab_web}. The Results of other methods are taken from their respective original papers. 

Table \ref{tab_an10} presents the comparison between ALR and the baseline method on ANIMAL-10, where ALR demonstrates a notable improvement in performance. Table \ref{tab_1M} compares the classification accuracy of ALR and the baseline method on Clothing1M. ALR outperforms all the comparison methods, including Co-teaching, which uses a dual network for sample selection. In table \ref{tab_web}, the results for ALR and the baseline methods on the (mini) WebVision dataset are listed. While ALR's performance is slightly inferior to that of ELR, this can be attributed to the significant class imbalance \citep{liang2024nonlocal} in the WebVision dataset. This imbalance leads to a scarcity of samples for certain categories, making it difficult for the model to effectively learn these classes during the early training phases. Although label quality improves over time, the model remains less confident in these underrepresented samples, ultimately hindering its ability to make accurate predictions. 

\subsection{Empirical Analysis}\label{subsec54}
To validate the ALR's effectiveness in mitigating the model’s tendency to overfit incorrect labels, we performed experiments on the CIFAR-10 dataset with 40\% symmetric noise. The experimental outcomes are illustrated in Figure \ref{fig1}.

Figure \ref{fig1} (a) and (b) illustrate the model's memorization behavior on clean and mislabeled samples using cross-entropy loss (CE). In contrast, Figure \ref{fig1} (c) and (d) show the model's memorization behavior on clean and mislabeled samples using ALR. It is clear that ALR significantly reduces the model's tendency to fit noisy labels compared to the CE loss. These findings highlight ALR's ability to alleviate overfitting to mislabeled data, thereby preventing performance decline and confirming its effectiveness in improving model robustness.

\begin{figure}[htbp]
    \centering
    \subfigure[CE]{\includegraphics[width=0.49\textwidth]{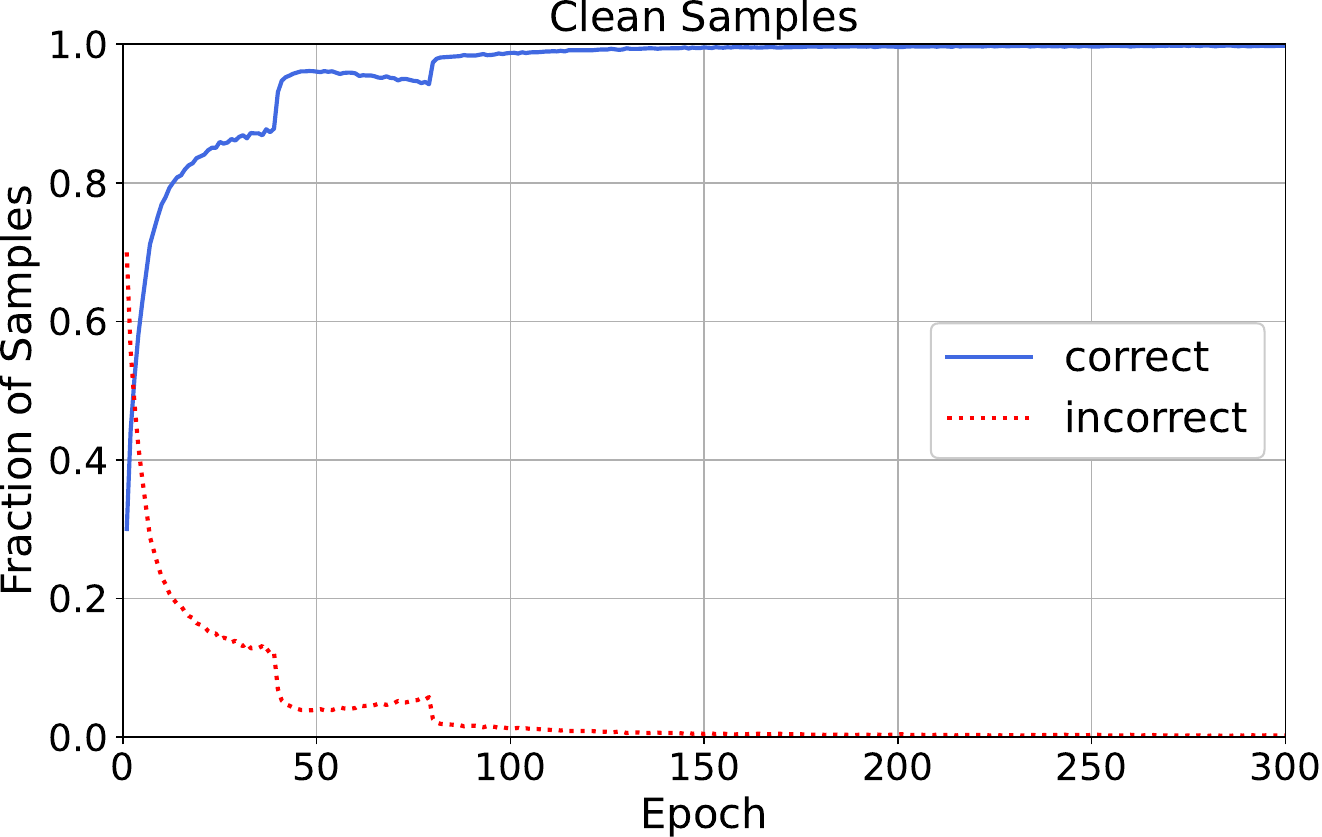}}
    \subfigure[CE]{\includegraphics[width=0.49\textwidth]{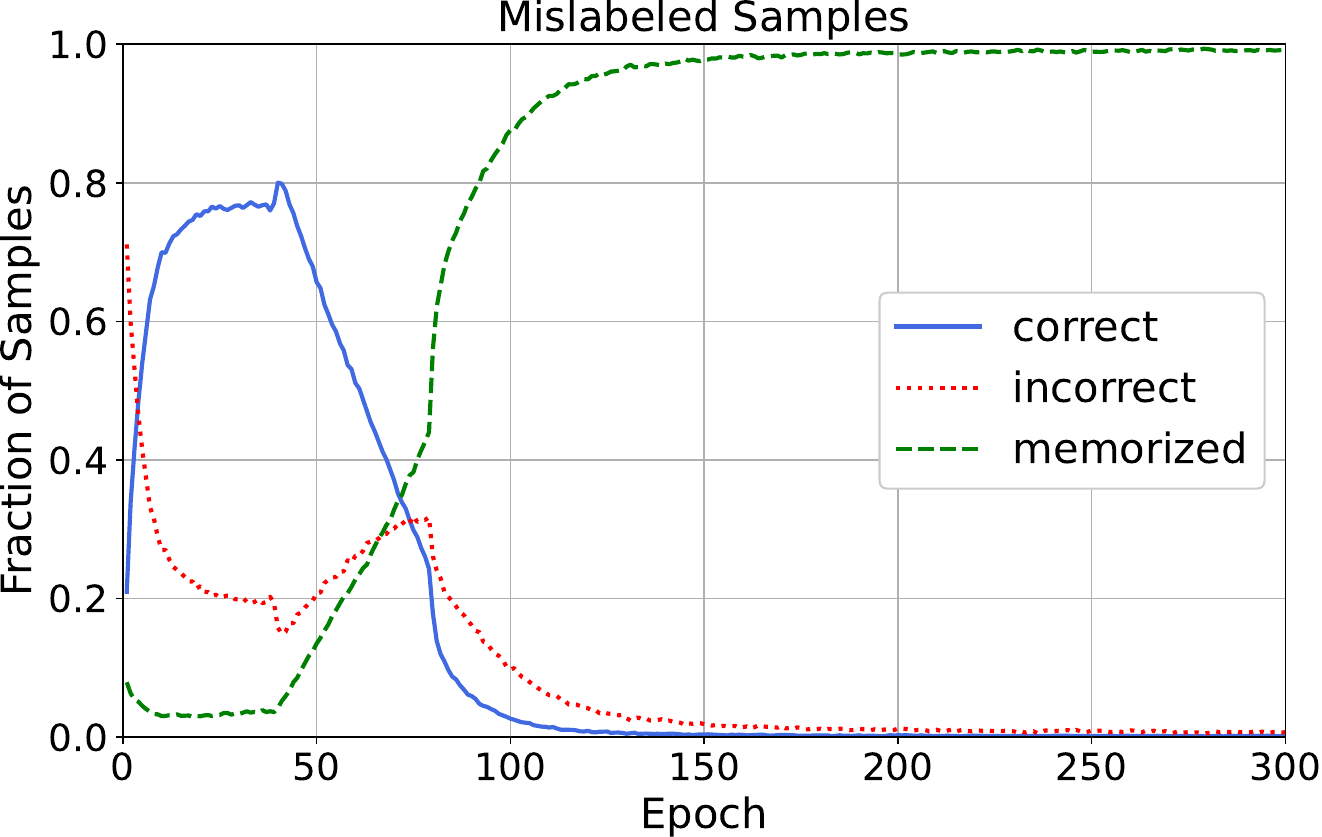}}
    \subfigure[ALR]{\includegraphics[width=0.49\textwidth]{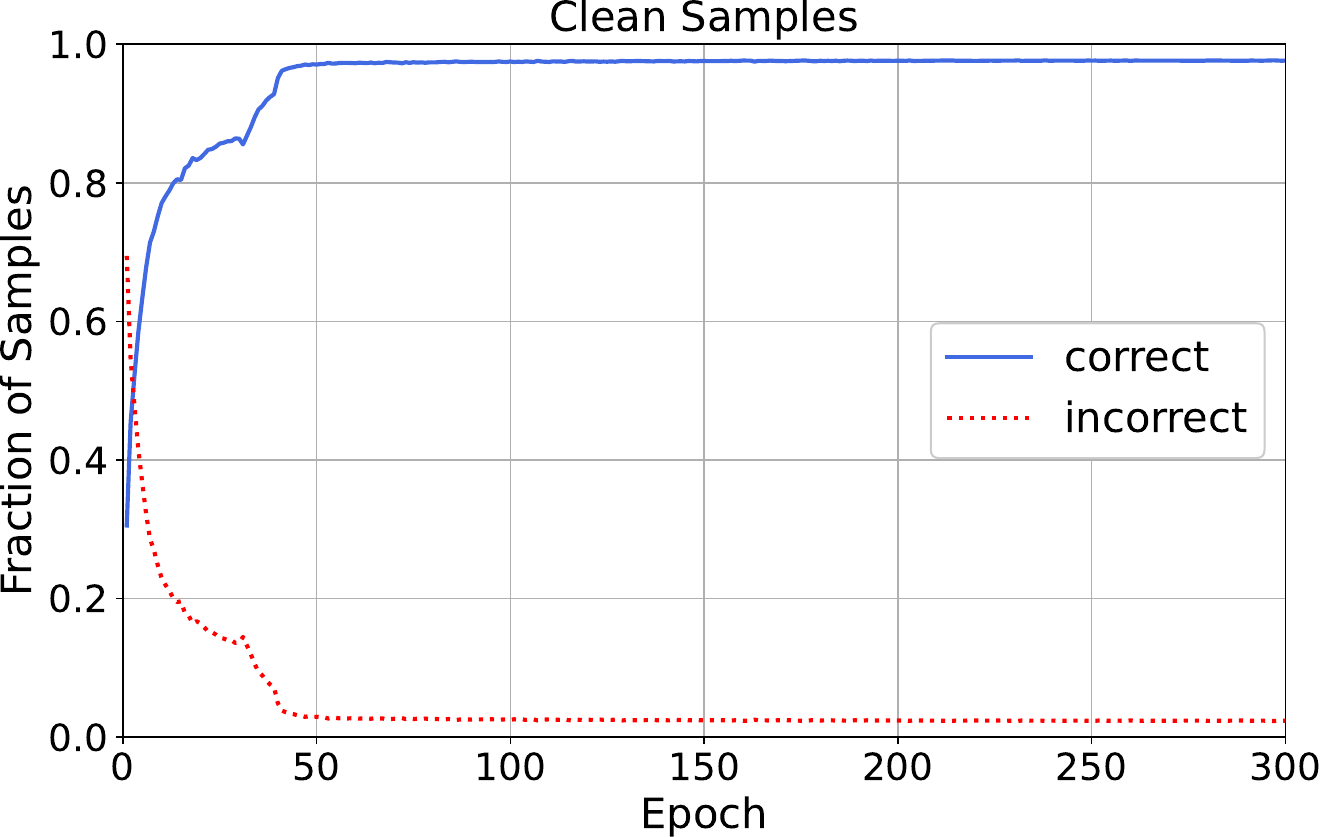}}
    \subfigure[ALR]{\includegraphics[width=0.49\textwidth]{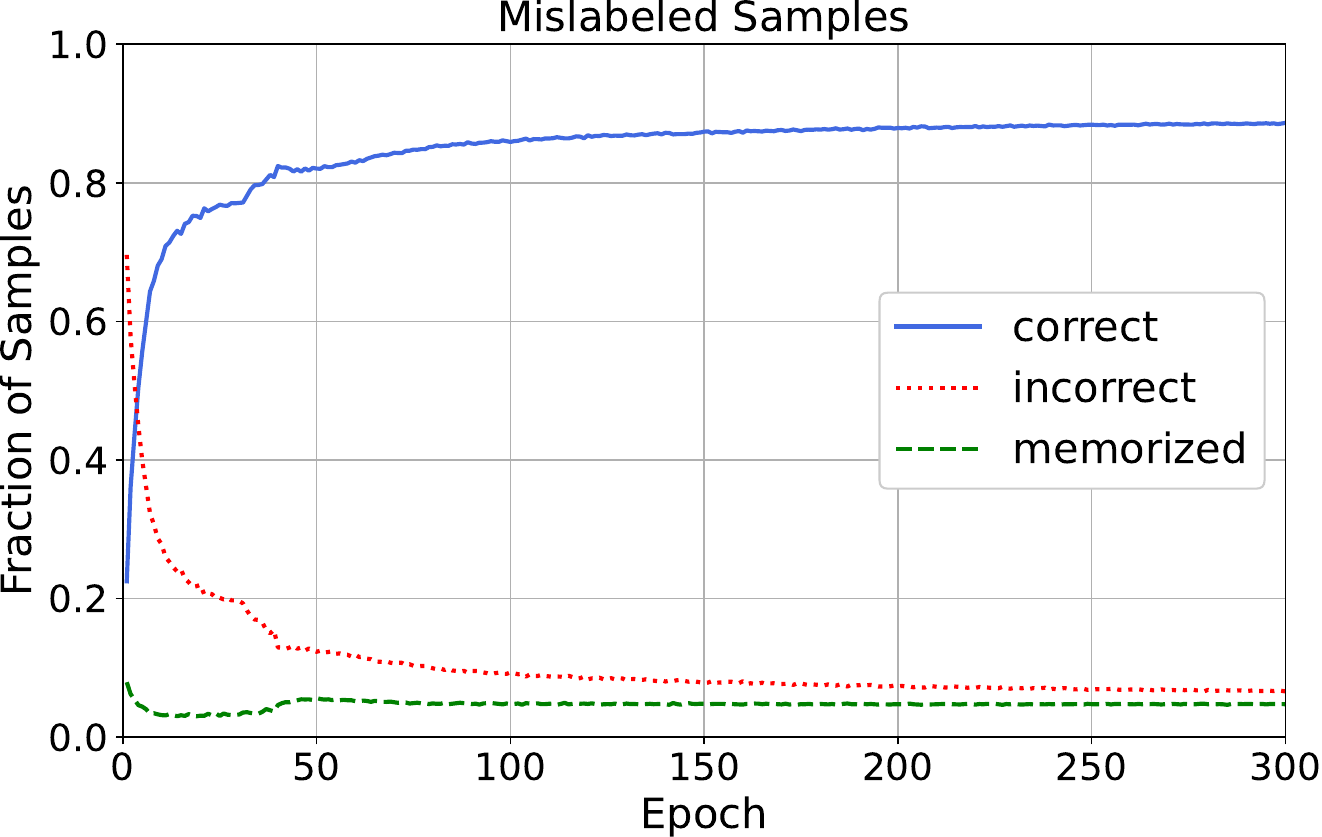}}

  \caption{We studied the memorization phenomenon of DNNs on CIFAR-10 using CE and ALR under 40\% symmetric noise. Plots (a) and (c) illustrate the proportion of correct predictions (in blue) and incorrect predictions (in red) on clean samples across different training phases. Plots (b) and (d) show the proportion of correct predictions (in blue), incorrect predictions (i.e., predictions that are neither equal to the true label nor the erroneous label, in red), and memorization (i.e., predictions that match the erroneous label, in green) on mislabeled samples across different training phases.}
    \label{fig1}
\end{figure}

\begin{figure}[htbp]
    \centering
    \subfigure[LR-Sym-40\%]{\includegraphics[width=0.32\textwidth]{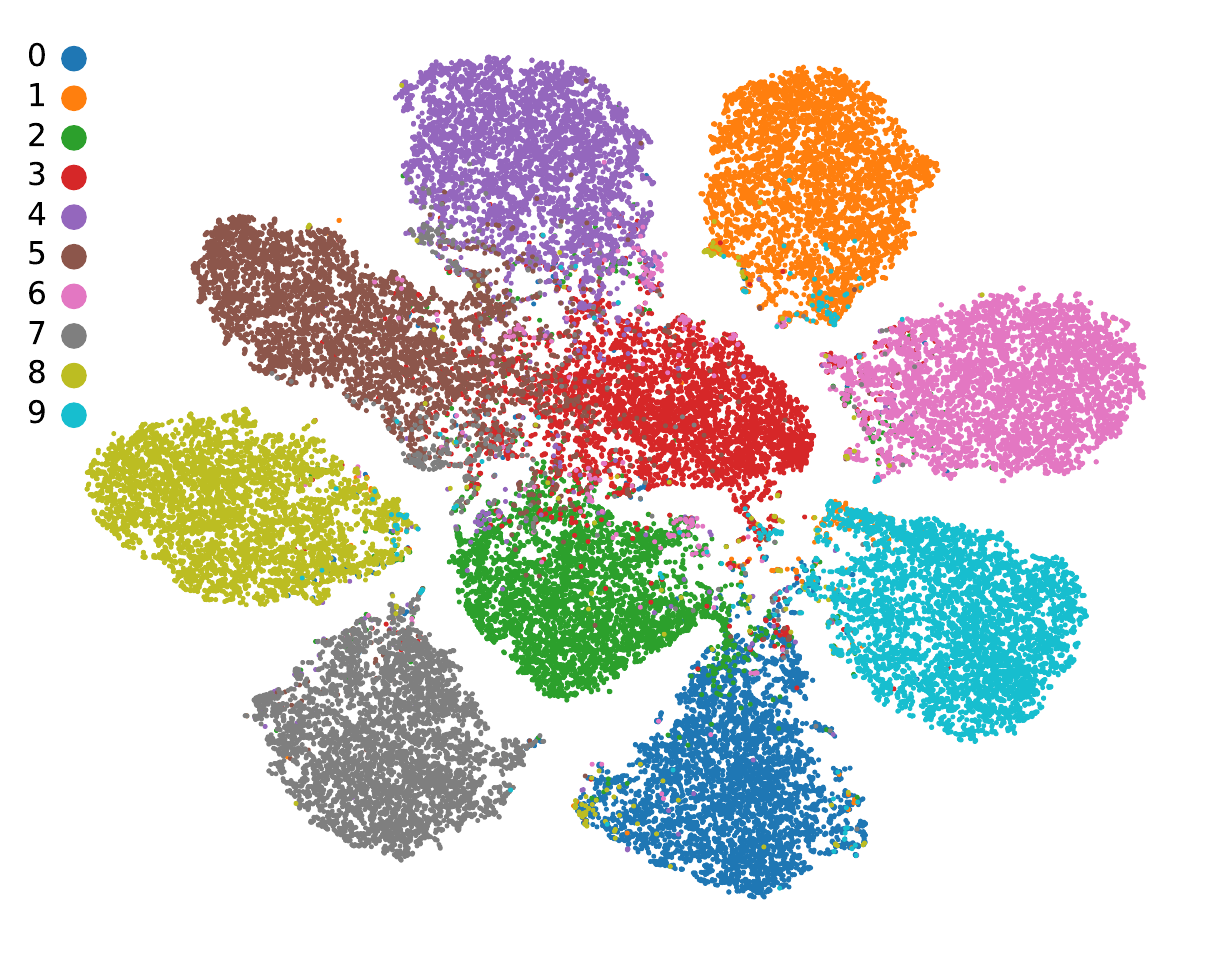}}
    \subfigure[LR-Sym-60\%]{\includegraphics[width=0.32\textwidth]{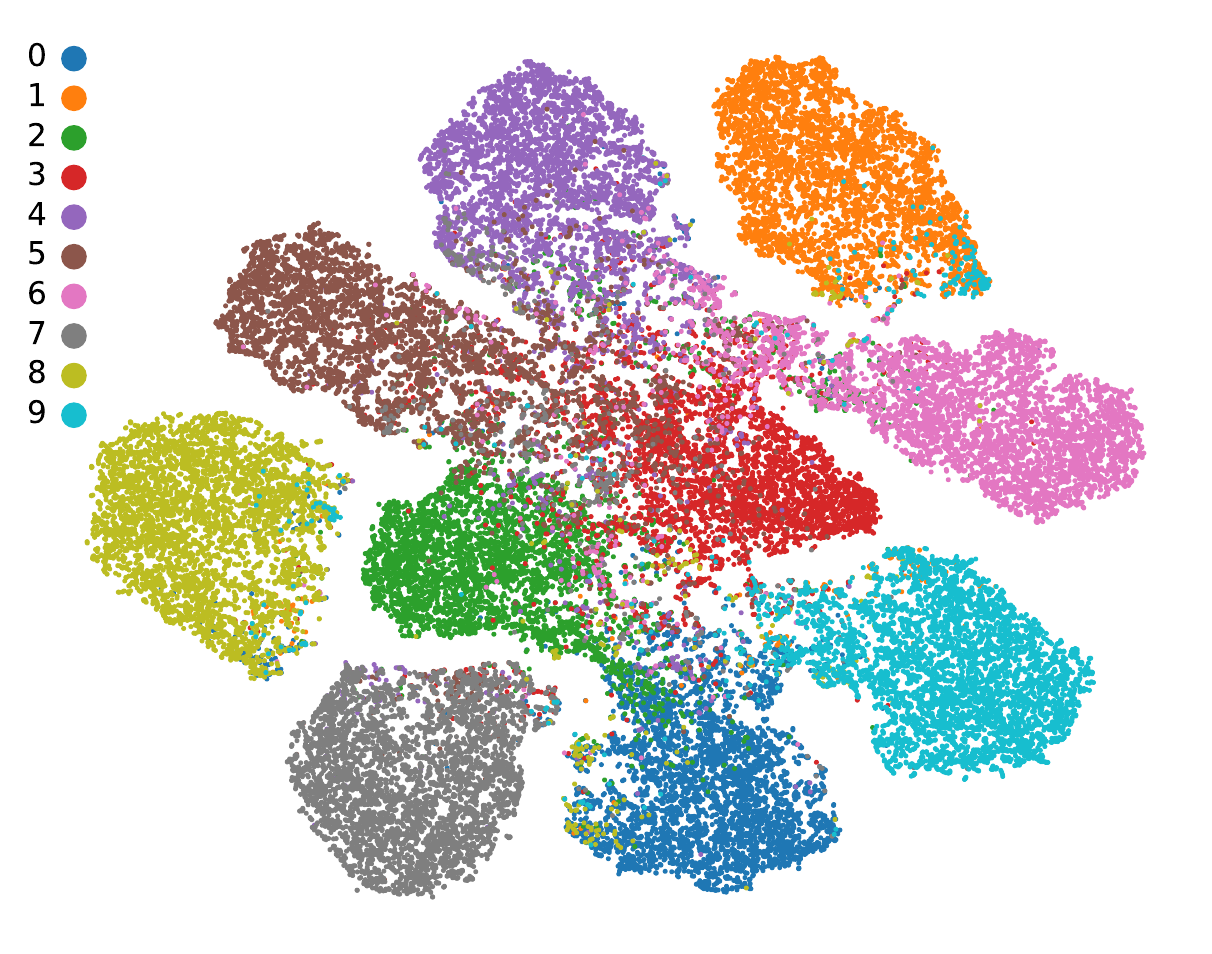}}
    \subfigure[LR-Asym-40\%]{\includegraphics[width=0.32\textwidth]{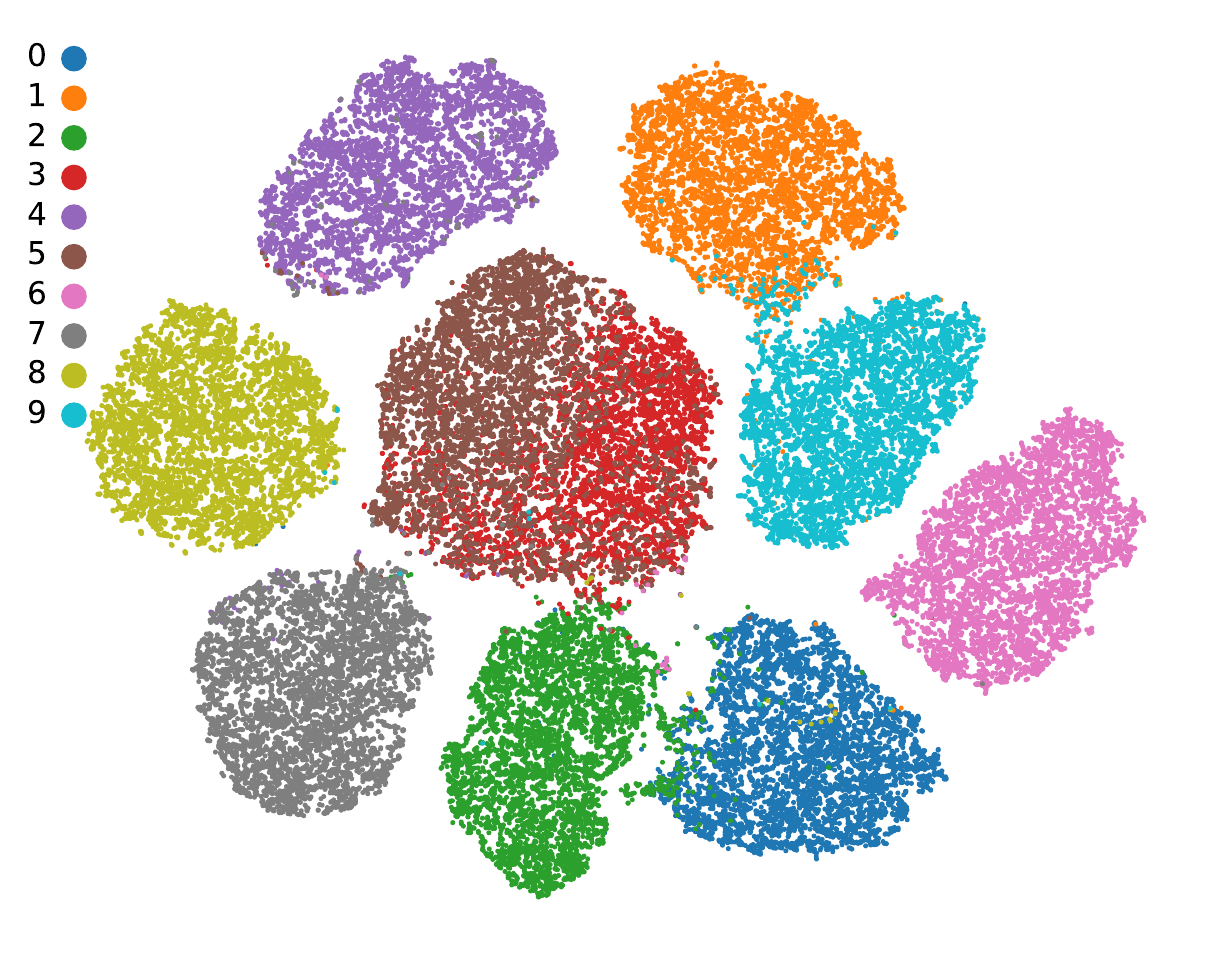}}
    \subfigure[ALR-Sym-40\%]{\includegraphics[width=0.32\textwidth]{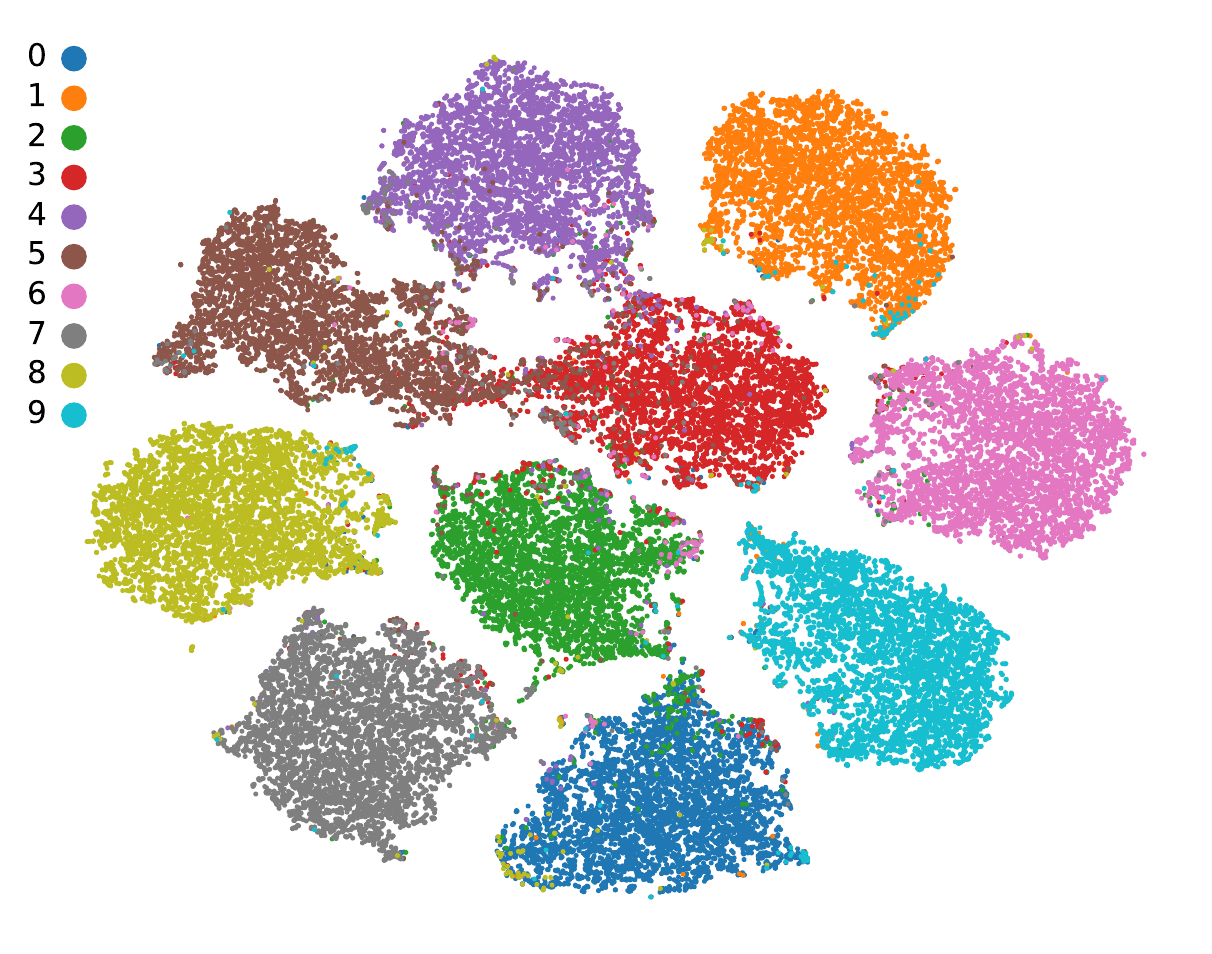}}
    \subfigure[ALR-Sym-60\%]{\includegraphics[width=0.32\textwidth]{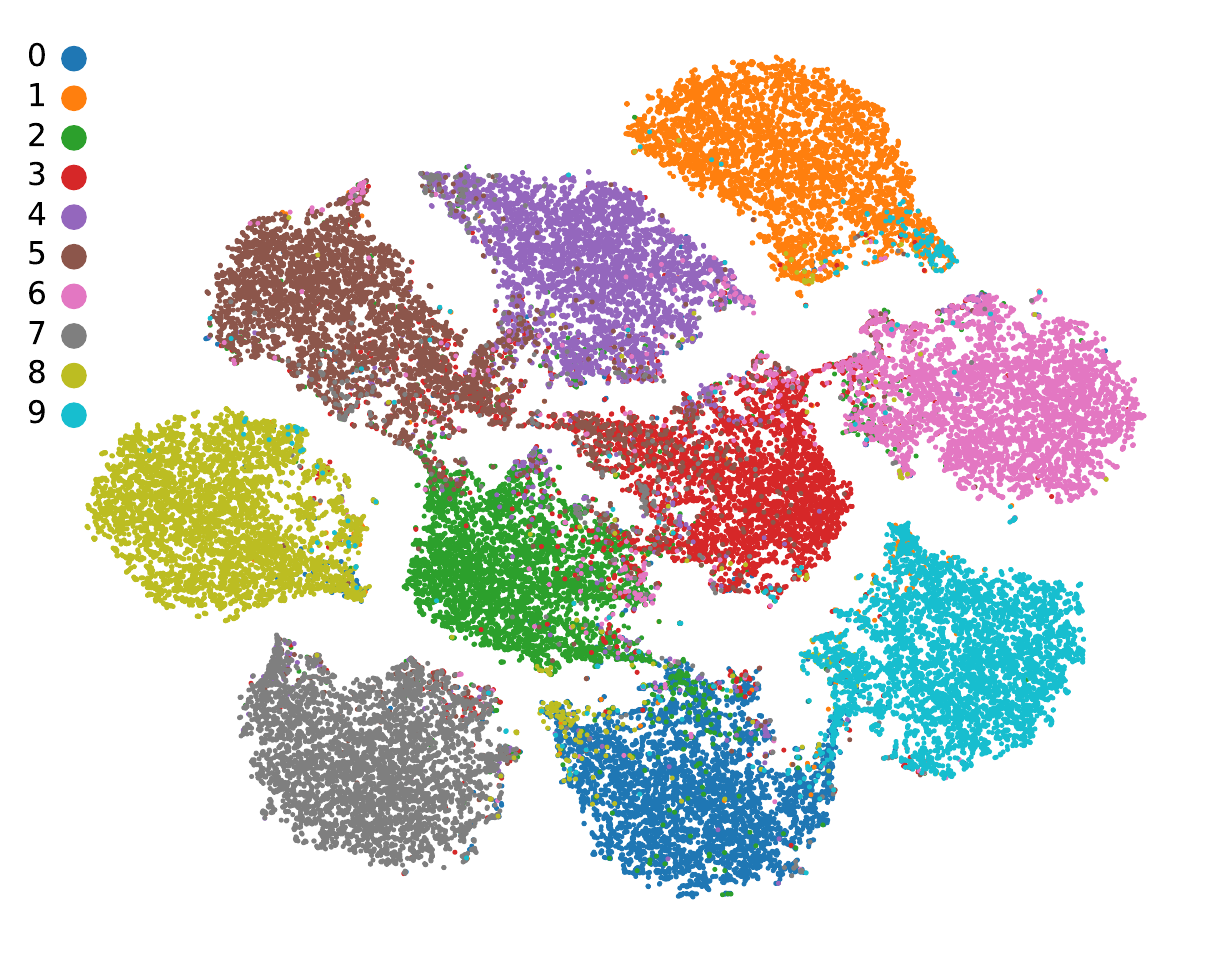}}
    \subfigure[ALR-Asym-40\%]{\includegraphics[width=0.32\textwidth]{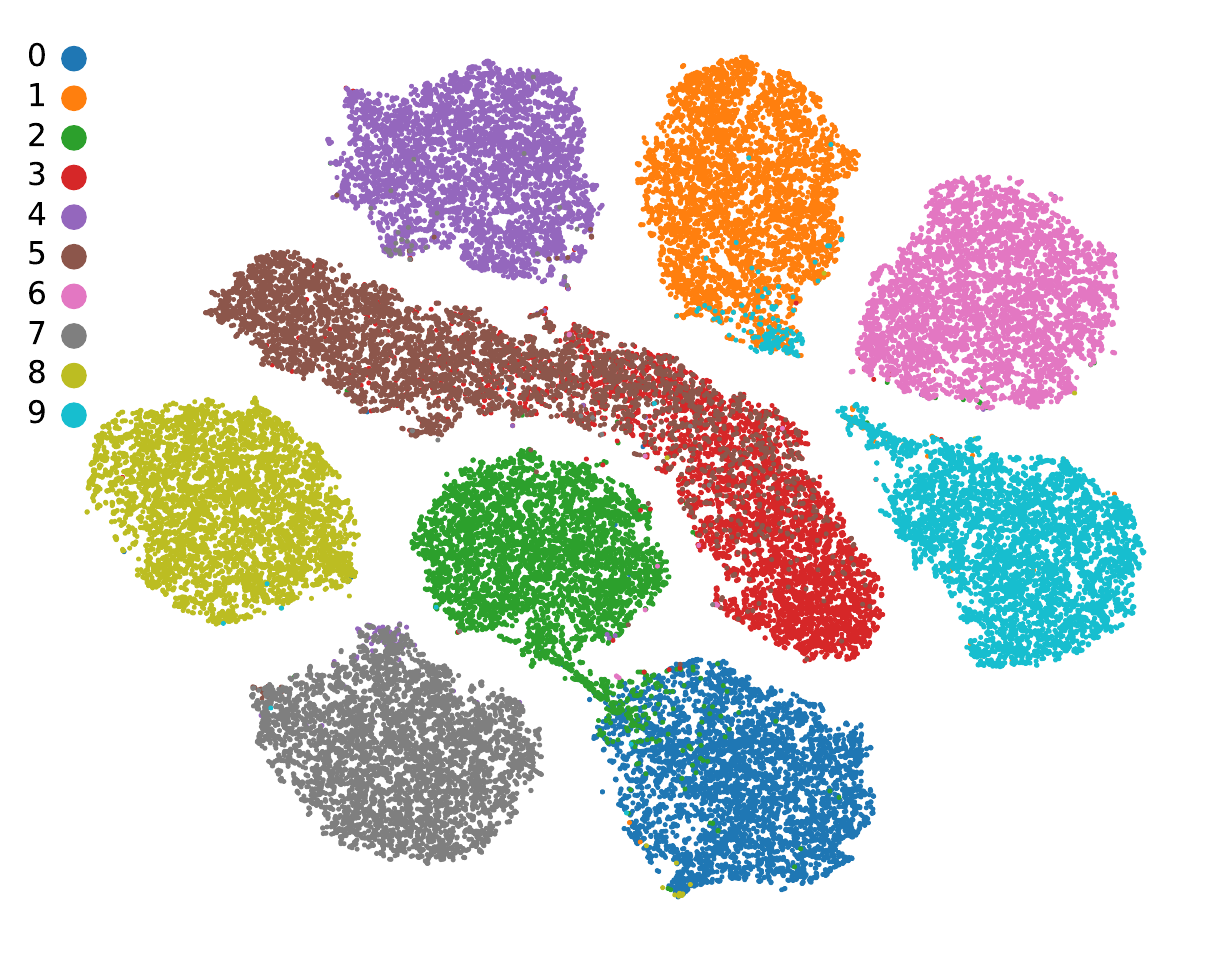}}

    \caption{2D t-SNE visualization of clean samples in the CIFAR-10 training set for the simplest label refurbishment method (LR) and ALR under different levels of label noise.}
    \label{fig2}
\end{figure}

\begin{figure}[htbp]
    \centering

    \subfigure[Sym-20\%]{\includegraphics[width=0.32\textwidth]{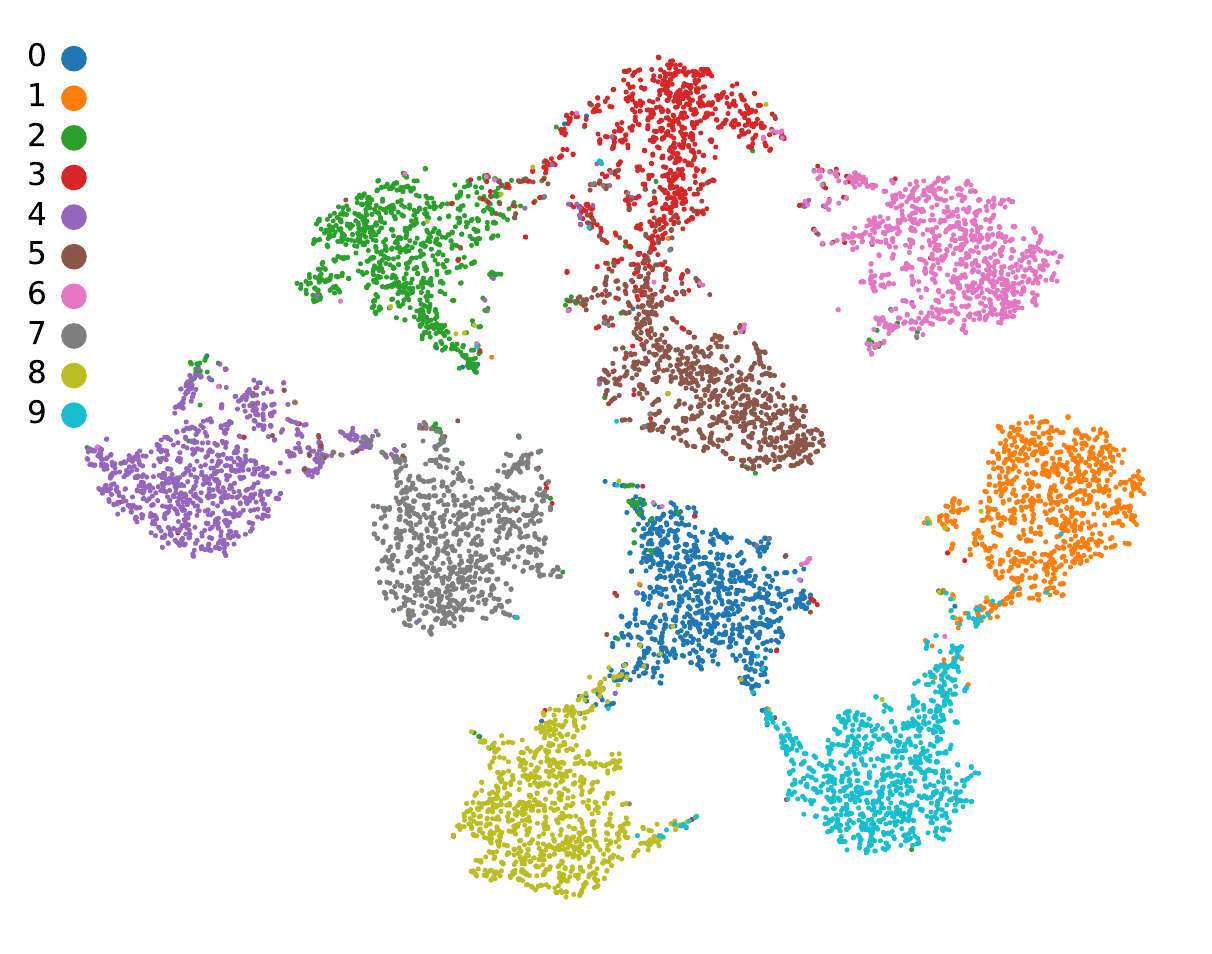}}
    \subfigure[Sym-40\%]{\includegraphics[width=0.32\textwidth]{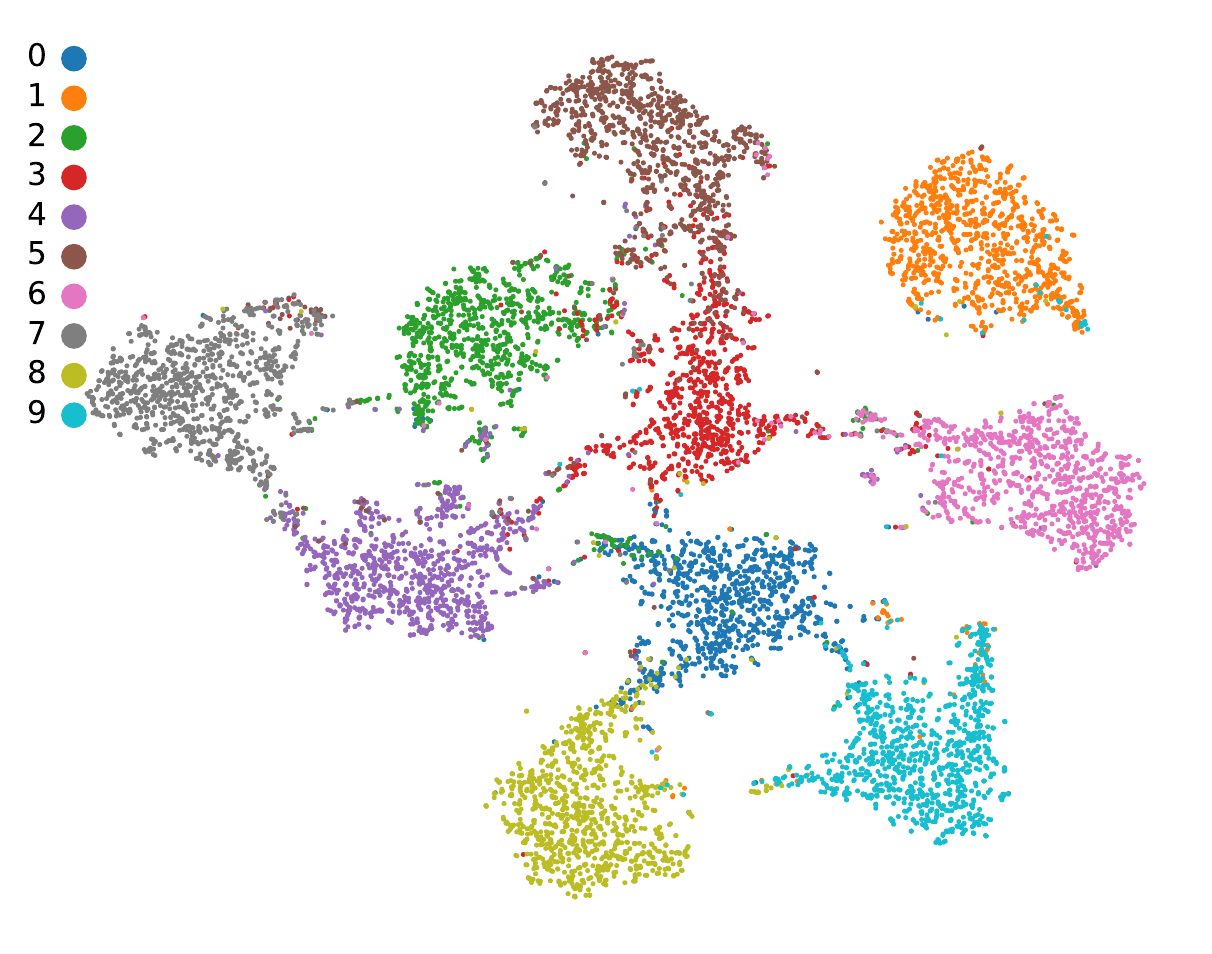}}
    \subfigure[Sym-60\%]{\includegraphics[width=0.32\textwidth]{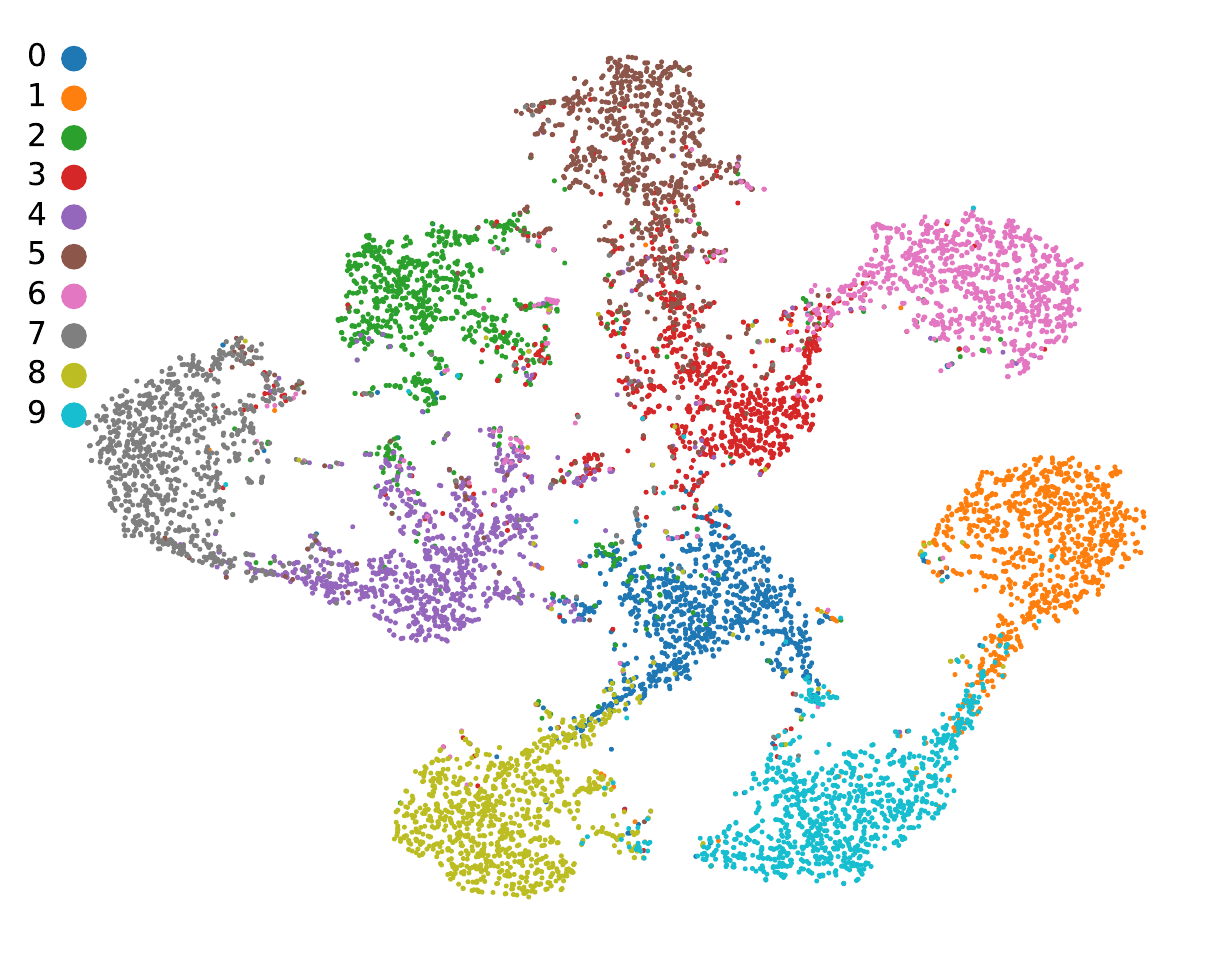}}
    \subfigure[Sym-80\%]{\includegraphics[width=0.32\textwidth]{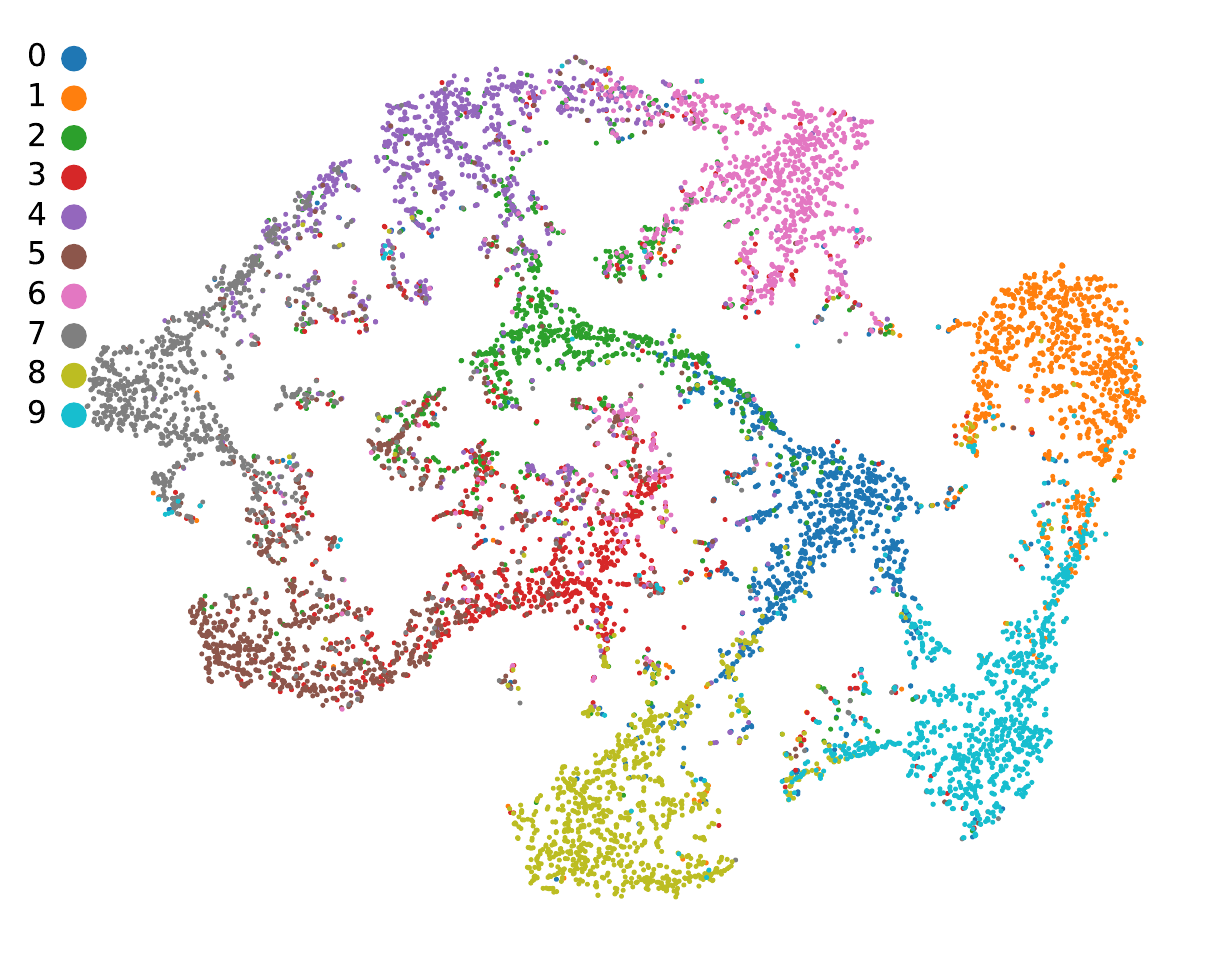}}
    \subfigure[Asym-20\%]{\includegraphics[width=0.32\textwidth]{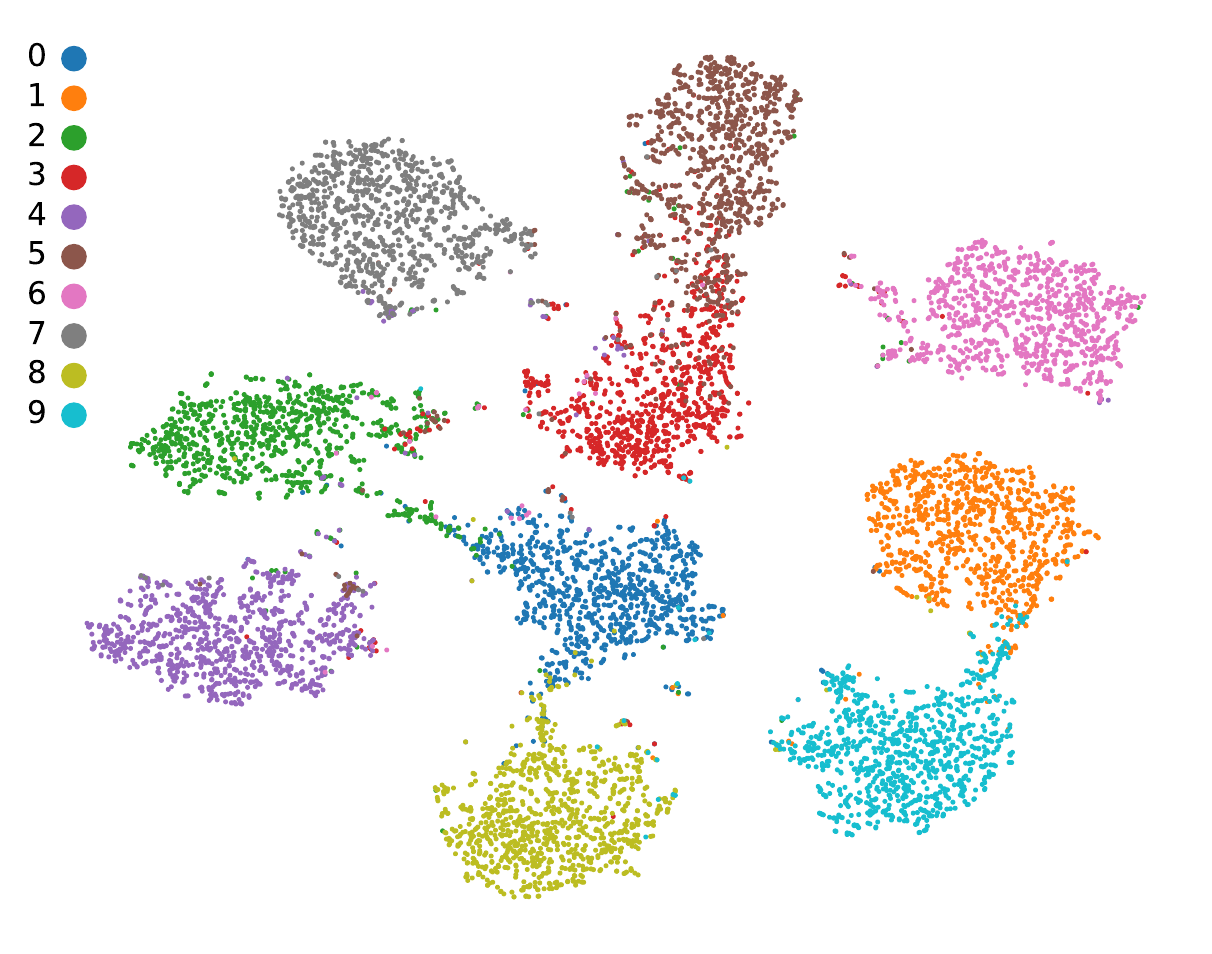}}
    \subfigure[Asym-40\%]{\includegraphics[width=0.32\textwidth]{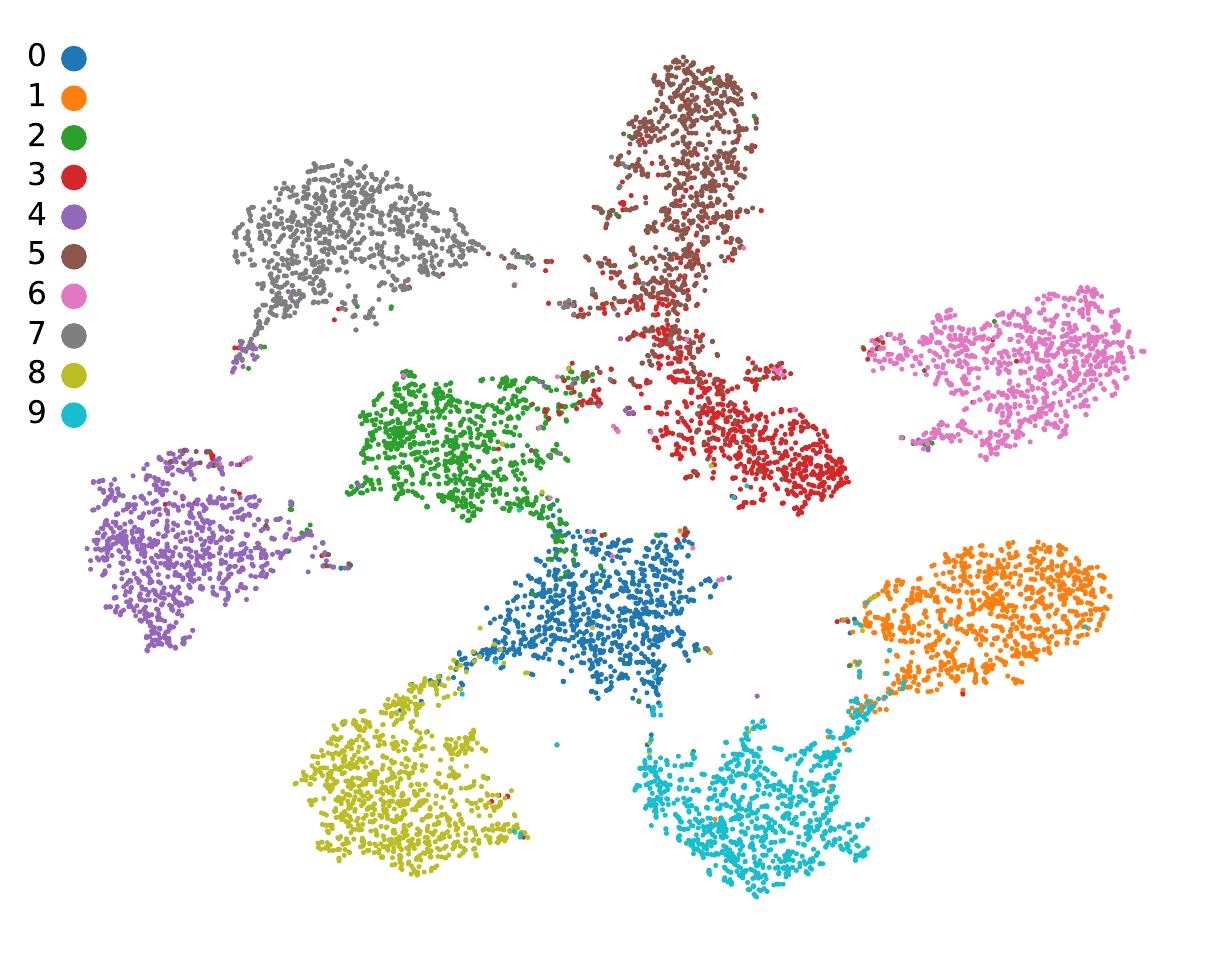}}

    \caption{Models were trained using ALR on the CIFAR-10 training set with different levels of symmetric and asymmetric label noise, and the clustering performance of these models on the CIFAR-10 test set was visualized using 2D t-SNE.}
    \label{fig3}
\end{figure}

To evaluate ALR's effectiveness in enhancing the model's fit to clean samples, we compare it with a basic label refinement method that uses only standard cross-entropy loss. This simplest label refinement (LR) method serves as a baseline for comparison with ALR under noisy label conditions. To this end, we performed experiments on the CIFAR-10 training set under three different noise settings: Sym-40\%, Sym-60\%, and Asym-40\%. 

In Figure \ref{fig2}, we use 2D t-SNE to visualize the fitting of clean samples in the training set by models using LR and ALR, respectively. By comparing the sample aggregation in Figure \ref{fig2} (a), (b), (c) and (d), (e), (f), it is evident that ALR facilitates closer aggregation of similar samples, while also enhancing the separation between different categories. This demonstrates that ALR effectively addresses the underfitting of clean samples, which is a common issue in label refurbishment methods.

We train ALR on the CIFAR-10 dataset with different levels of label noise and used 2D t-SNE to visualize its clustering performance on the test set, as illustrated in Figure \ref{fig3}. The visualizations show that ALR preserves well-defined separations between categories, even when the training data contains noisy labels. These results demonstrate ALR’s ability to effectively mitigate overfitting to incorrect labels while enhancing its learning from clean samples.

\subsection{Ablation Experiments}\label{subsec55}
We perform ablation studies on multiple datasets with different noise conditions to gain deeper insights into the contributions of the two components of ALR. Removing the entropy loss regularization term is denoted as ``w/o EL,'' and removing both the entropy loss regularization term and label refurbishment is denoted as ``w/o EL+LR.'' Table \ref{tab_ablation10} and \ref{tab_ablation100} present the classification accuracy of ALR and its two variants on CIFAR-10 and CIFAR-100 under varying levels of label noise. Both tables indicate that removing the entropy loss regularization term results in a performance decline, highlighting its critical role in enhancing the model’s fit to clean samples.

\begin{table}[htbp]
    \caption{Test classification accuracy from ablation study results on ALR, conducted on the CIFAR-10 with label noise.}\label{tab_ablation10}
    \resizebox{\textwidth}{!}{
    \begin{tabular}{@{}lccccccc@{}}
    \toprule
    & \multicolumn{7}{c}{\textbf{CIFAR-10}} \\
    \midrule
    \multirow{2}{*}{\textbf{Method}} & \multirow{2}{*}{\textbf{No noise}} & \multicolumn{4}{c}{\textbf{Symm}} & \multicolumn{2}{c}{\textbf{Asymm}} \\
    \cmidrule(lr){3-6} \cmidrule(l){7-8}
    & & 20\% & 40\% & 60\% & 80\% & 20\% & 40\% \\
    \midrule
    ALR& 95.07$\pm$0.14 &93.65$\pm$0.09 & 91.89$\pm$0.20 & 89.00$\pm$0.23 & 78.47$\pm$0.6 & 94.40$\pm$0.14 & 92.84$\pm$0.07 \\
    ALR w/o EL& 94.87$\pm$0.16 & 93.35$\pm$0.11 & 90.89$\pm$0.09 & 87.45$\pm$0.36 & 74.24$\pm$0.21 & 93.98$\pm$0.08 & 90.82$\pm$0.08 \\
    ALR w/o EL+LR& 94.85$\pm$0.11 & 84.43$\pm$0.35 & 66.88$\pm$0.56 & 45.55$\pm$0.98 & 26.29$\pm$1.22 & 87.53$\pm$0.39 & 76.83$\pm$0.76 \\
    \bottomrule
\end{tabular}
}
\end{table}

\setlength{\tabcolsep}{3pt} 
\begin{table}[htbp]
    \caption{Test classification accuracy from ablation study results on ALR, conducted on the CIFAR-100 with label noise.}\label{tab_ablation100}
    \resizebox{\textwidth}{!}{
    \begin{tabular}{@{}lccccccc@{}}
    \toprule
    & \multicolumn{7}{c}{\textbf{CIFAR-100}} \\
    \midrule
    \multirow{2}{*}{\textbf{Method}} & \multirow{2}{*}{\textbf{No noise}} & \multicolumn{4}{c}{\textbf{Symm}} & \multicolumn{2}{c}{\textbf{Asymm}} \\
    \cmidrule(lr){3-6} \cmidrule(l){7-8}
    & & 20\% & 40\% & 60\% & 80\% & 20\% & 40\% \\
    \midrule
    ALR & 78.10$\pm$0.35 & 74.50$\pm$0.19 & 69.80$\pm$0.38 & 62.56$\pm$0.36 & 31.96$\pm$0.54 & 77.06$\pm$0.13 & 75.02$\pm$0.18 \\
    ALR w/o EL & 77.72$\pm$0.37 & 73.12$\pm$0.21 & 67.88$\pm$0.31 & 59.62$\pm$0.68 & 31.73$\pm$0.50 & 76.24$\pm$0.25 & 67.85$\pm$0.68 \\
    ALR w/o EL+LR & 76.28$\pm$0.17 & 63.85$\pm$0.21 & 49.42$\pm$0.75 & 31.10$\pm$0.19 & 9.13$\pm$0.14 & 63.85$\pm$0.16 & 46.53$\pm$0.13 \\
    \bottomrule
\end{tabular}
}
\end{table}
Furthermore, when label refurbishment is removed, the model effectively reverts to a conventional approach that relies solely on cross-entropy loss. The results demonstrate that eliminating label refurbishment leads to a significant decline in performance, underscoring its crucial role in preventing the fitting of incorrectly labeled samples.

Remarkably, even without label noise, ALR achieves the highest classification accuracy, surpassing the model trained solely with cross-entropy loss. This enhancement is attributed to ALR's use of soft labels as the learning target, which helps mitigate overfitting to low-quality images. These findings highlight the robustness of ALR, demonstrating its ability to outperform traditional cross-entropy loss across both clean datasets and those with label noise.

\subsection{Parameter Sensitivity}\label{subsec56}

\subsubsection{Sensitivity Analysis of the weight coefficient \texorpdfstring{$\alpha$}{alpha}}\label{subsubsec1}

In ALR, the hyperparameter $\alpha$ controls the entropy loss regularization. A large $\alpha$ can cause overconfident predictions and lead to "lazy" learning. Therefore, it is advisable to set $\alpha$ to a smaller value.

To investigate the impact of the hyperparameter $\alpha$, we perform 
experiments on CIFAR-10 and CIFAR-100 under 40\% symmetric and asymmetric noise. Figure \ref{fig_alpha} presents the results of experiments exploring a range of $\alpha$ values, including {0.1, 0.2, 0.5, 1.0}. The findings reveal that changes in the $\alpha$ value have a negligible effect on model performance across various datasets and noise conditions. This robustness underscores the method's practicality and broad applicability. Consequently, we set $\alpha$ to 0.2 for all experiments, including the three real-world noisy datasets. 

\begin{figure}[htbp]
    \centering
    \subfigure[C10 Sym-40\%]{\includegraphics[width=0.24\textwidth]{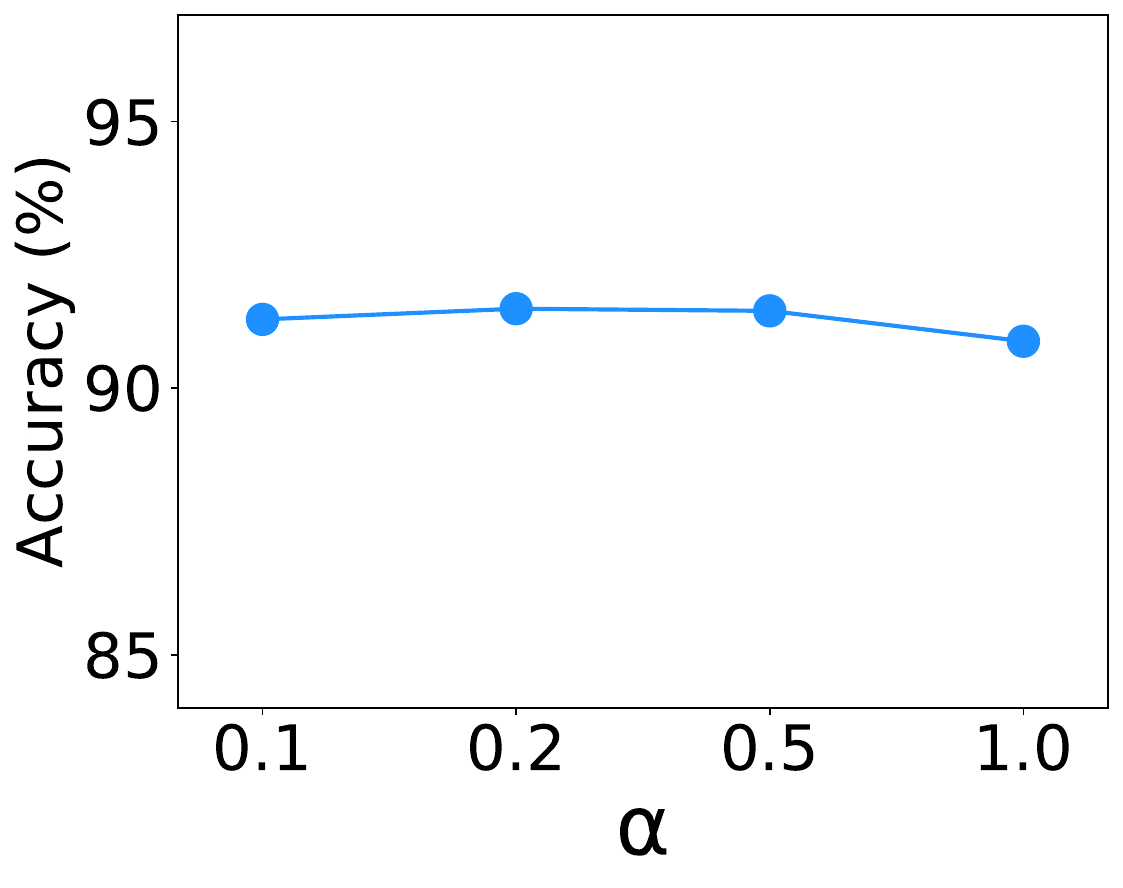}}
    \subfigure[C10 Asym-40\%]{\includegraphics[width=0.24\textwidth]{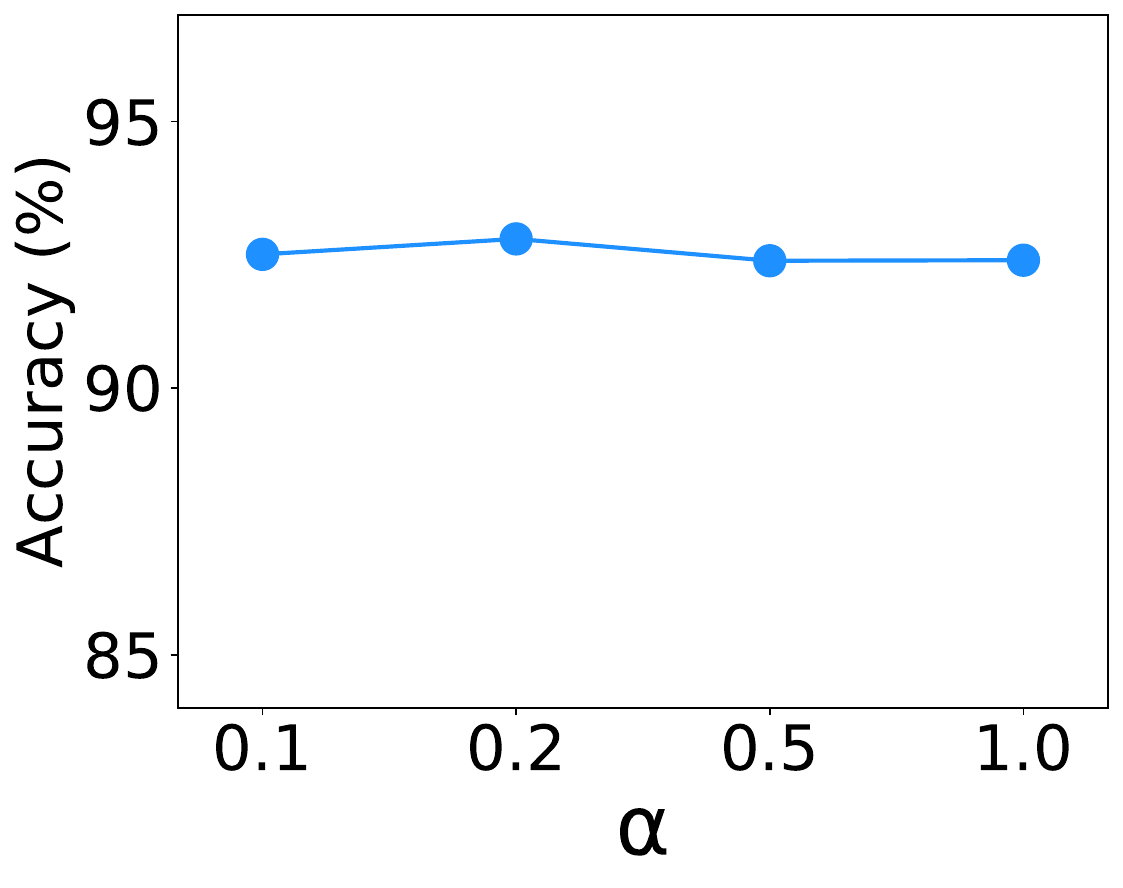}}
    \subfigure[C100 Sym-40\%]{\includegraphics[width=0.24\textwidth]{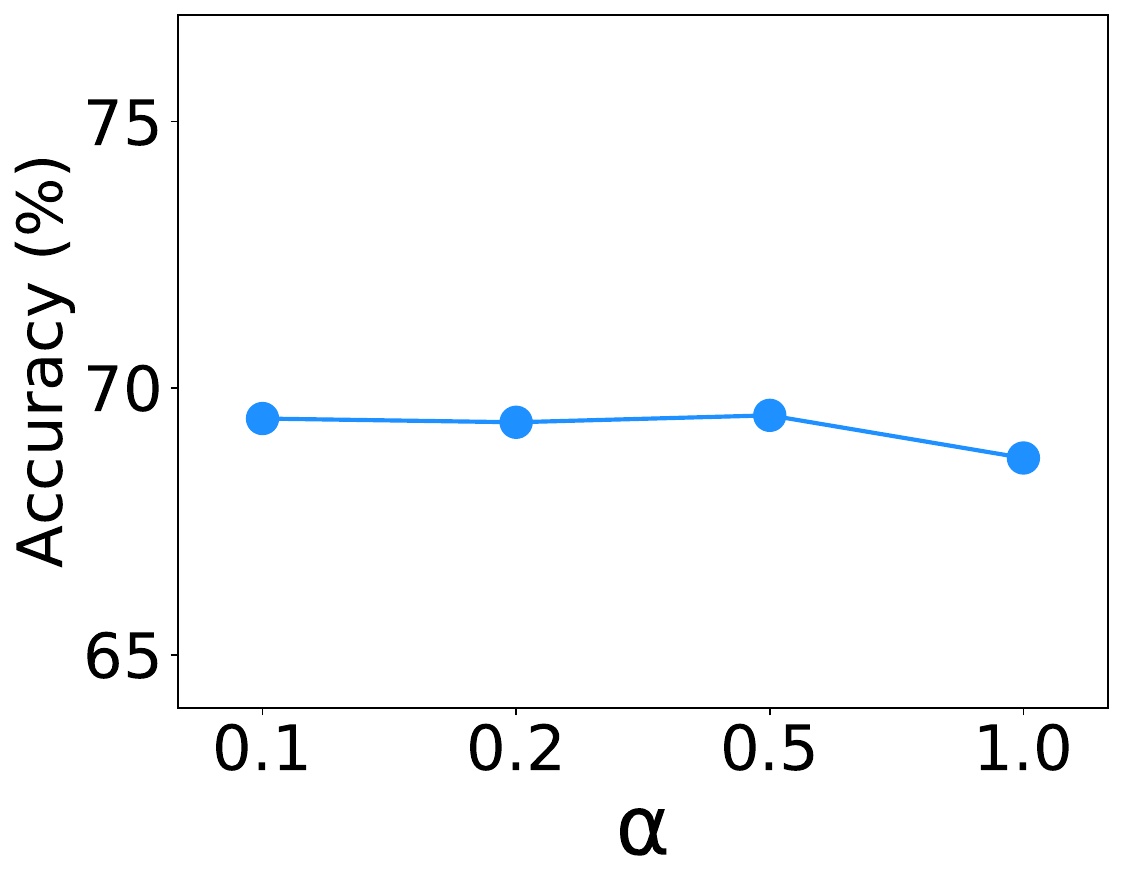}}
    \subfigure[C100 Asym-40\%]{\includegraphics[width=0.241\textwidth]{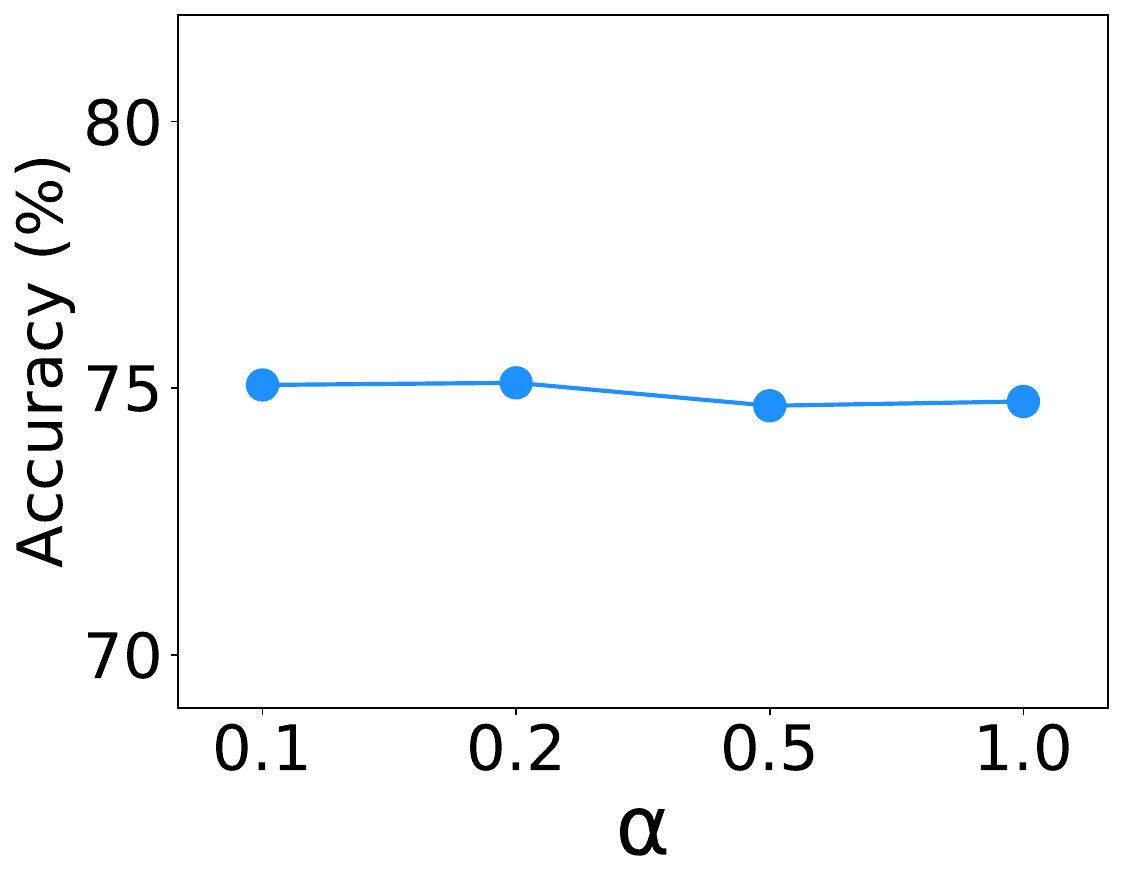}}
    
    \caption{Impact of the entropy loss regularization weight coefficient $\alpha$ on test classification accuracy on CIFAR-10 (C10) and CIFAR-100 (C100) under 40\% symmetric and asymmetric label noise.}
    \label{fig_alpha}
\end{figure}

\subsubsection{Sensitivity Analysis of the warm-up phase (\texorpdfstring{$m$}{m} epochs)}\label{subsubsec2}

The training process begins with a warm-up phase lasting $m$ epochs, which is designed to establish a reasonably well-trained model before initiating label refinement. If the warm-up duration is too short,  the model’s predictions may become unreliable. Conversely, an excessively long warm-up phase increases the risk of overfitting to noisy data. To examine the impact of $m$ on model performance, we conducted sensitivity experiments on CIFAR-10 and CIFAR-100 under 40\% symmetric and asymmetric noise conditions. As the warm-up phase typically takes place early in training, we evaluated several $m$ values before the first scheduled learning rate adjustment at epoch 40, specifically testing $\{20, 25, 30, 35, 40\}$.

The results in Figure \ref{fig_m} show that model performance remains largely stable despite changes in $m$ across different datasets and noise conditions. This resilience is likely due to ALR’s capacity to prioritize the learning of clean samples in the later stages of training, effectively minimizing performance differences associated with different $m$ values.

\begin{figure}[htbp]
    \centering
    \subfigure[C10 Sym-40\%]{\includegraphics[width=0.24\textwidth]{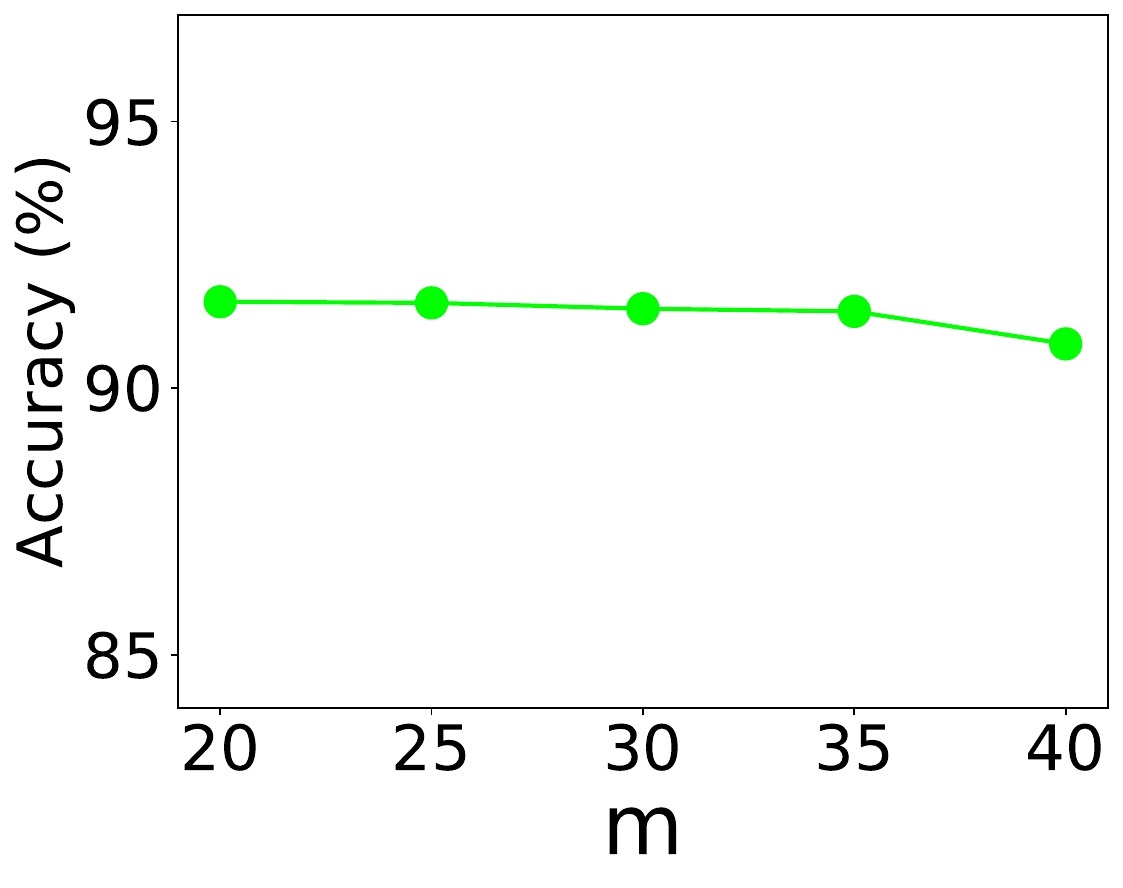}}
    \subfigure[C10 Asym-40\%]{\includegraphics[width=0.24\textwidth]{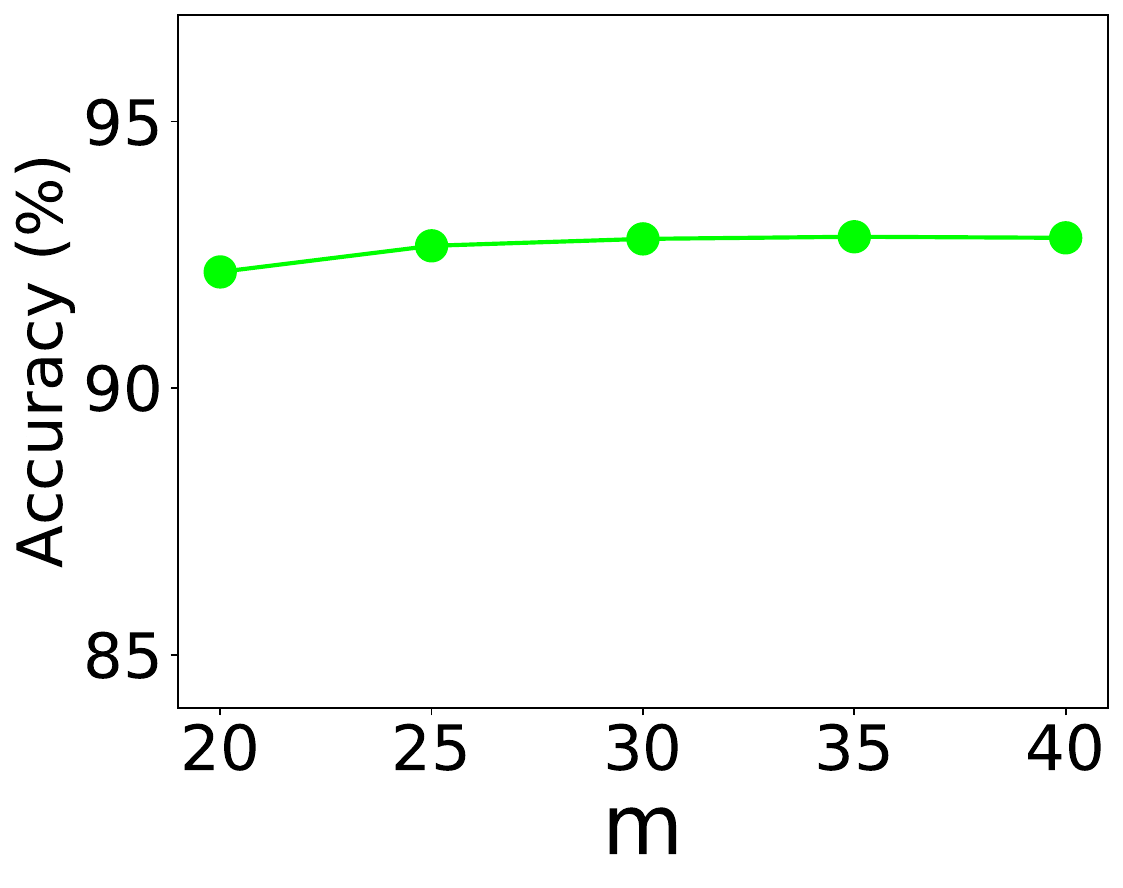}}
    \subfigure[C100 Sym-40\%]{\includegraphics[width=0.24\textwidth]{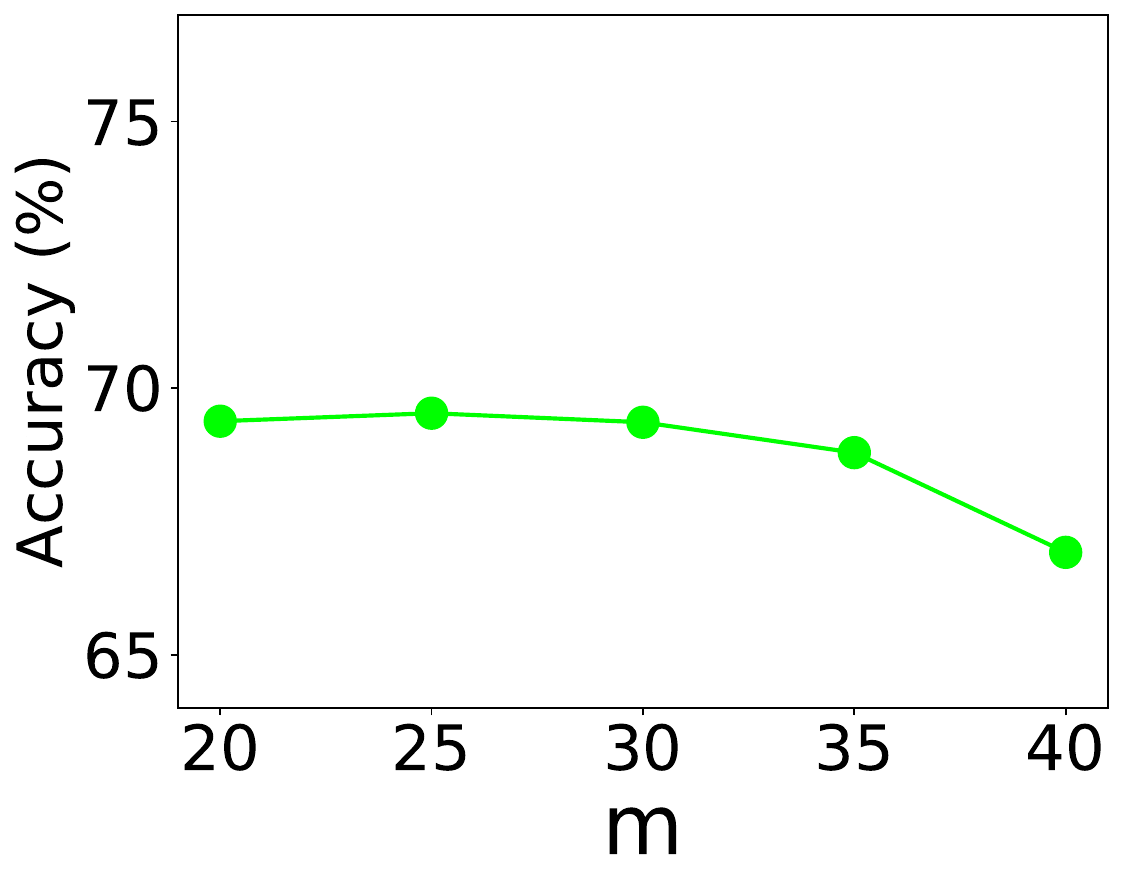}}
    \subfigure[C100 Asym-40\%]{\includegraphics[width=0.24\textwidth]{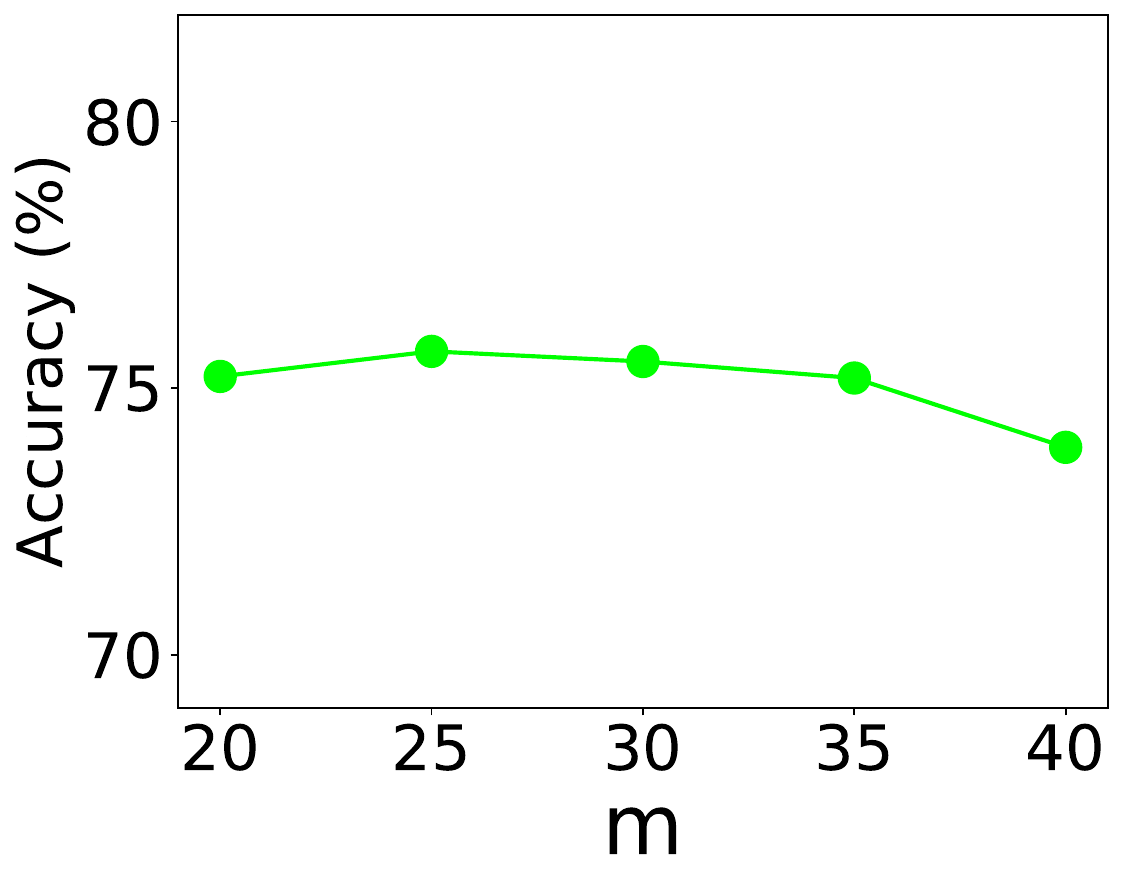}}
    \caption{Impact of the momentum hyperparameter $m$ on test classification accuracy on CIFAR-10 (C10) and CIFAR-100 (C100) under 40\% symmetric and asymmetric label noise.}
    \label{fig_m}
\end{figure}

\subsubsection{Sensitivity Analysis of the momentum hyperparameter \texorpdfstring{$\lambda$}{lambda}}\label{subsubsec3}
To evaluate the influence of the momentum parameter $\lambda$ in target estimation on model performance, we perform sensitivity analyses under 40\% symmetric and asymmetric noise conditions on CIFAR-10 and CIFAR-100. The parameter $\lambda$ was tested with values ${0.7, 0.8, 0.9, 0.95}$, and the corresponding results are illustrated in Figure \ref{fig_lambda}.

Figure \ref{fig_lambda} reveals that CIFAR-10 is less sensitive to changes in $\lambda$ compared to CIFAR-100. For CIFAR-10, model performance remains stable across different noise types, while for CIFAR-100, performance declines notably when $\lambda$ approaches 0.95, particularly under asymmetric noise. This indicates that excessive momentum can negatively impact learning in more complex datasets. Overall, $\lambda = 0.9$ achieves a good balance, ensuring robust performance across various datasets and noise conditions. Therefore, we fix $\lambda$ at 0.9 for all datasets.

\begin{figure}[htbp]
    \centering
    \subfigure[C10 Sym-40\%]{\includegraphics[width=0.24\textwidth]{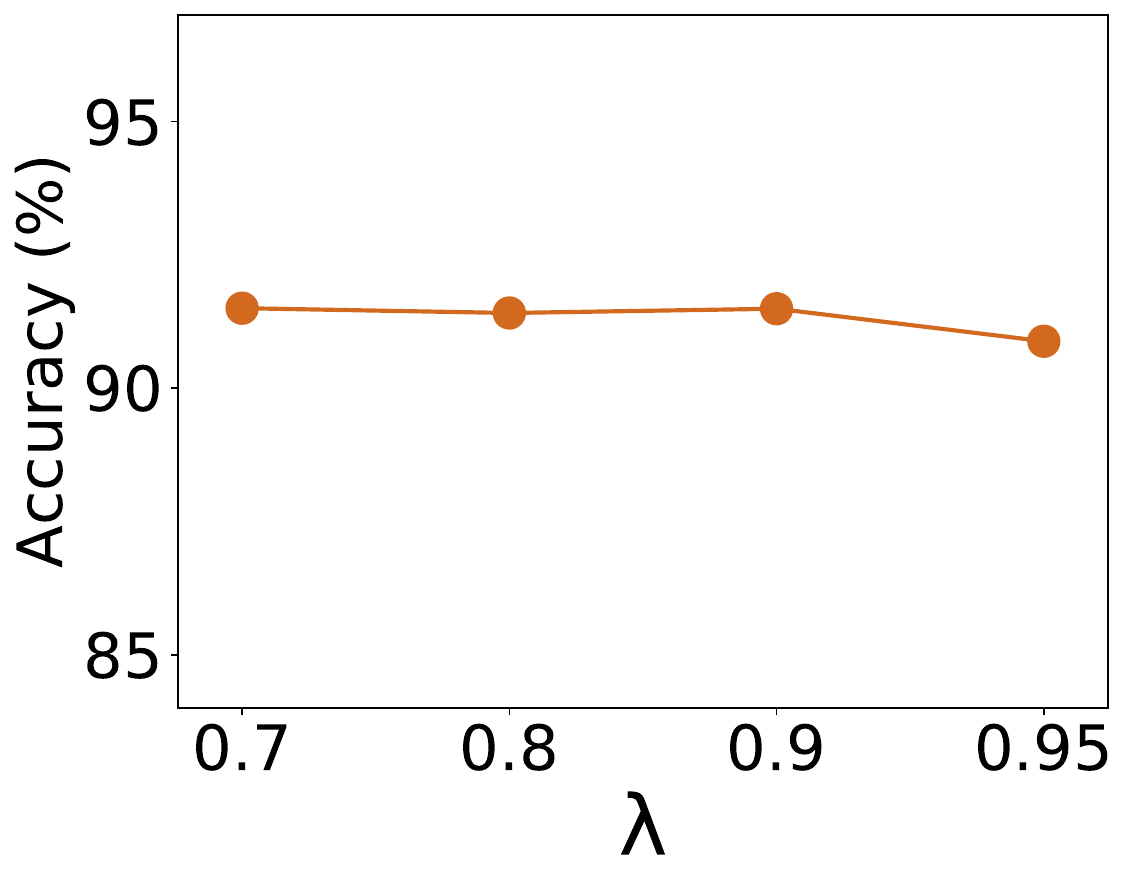}}
    \subfigure[C10 Asym-40\%]{\includegraphics[width=0.24\textwidth]{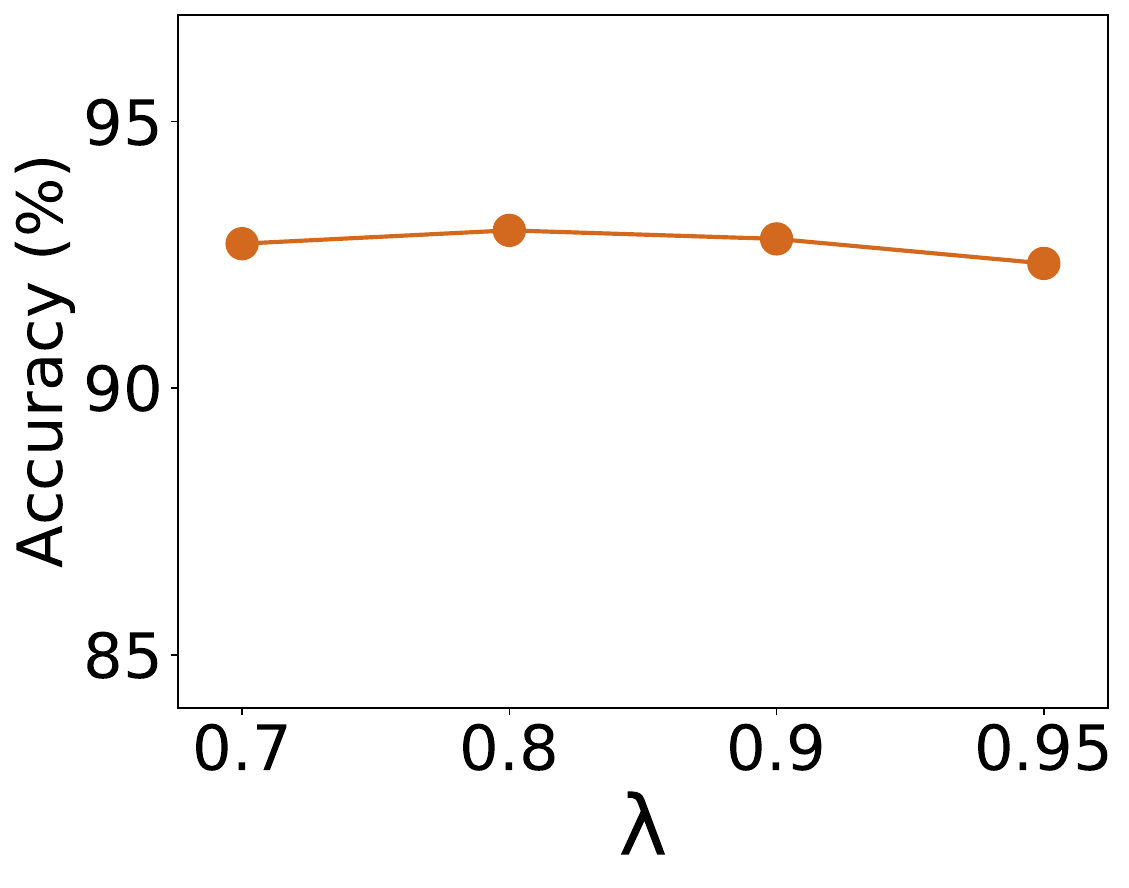}}
    \subfigure[C100 Sym-40\%]{\includegraphics[width=0.24\textwidth]{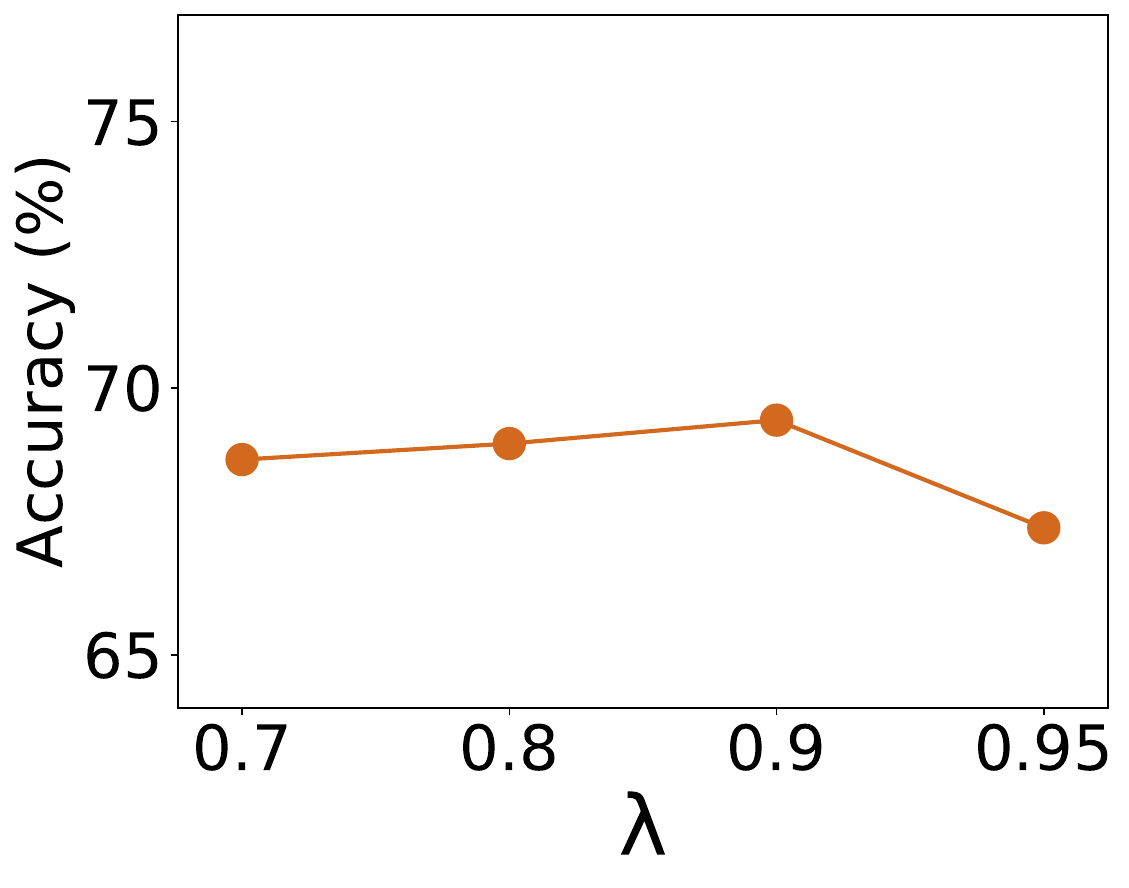}}
    \subfigure[C100 Asym-40\%]{\includegraphics[width=0.24\textwidth]{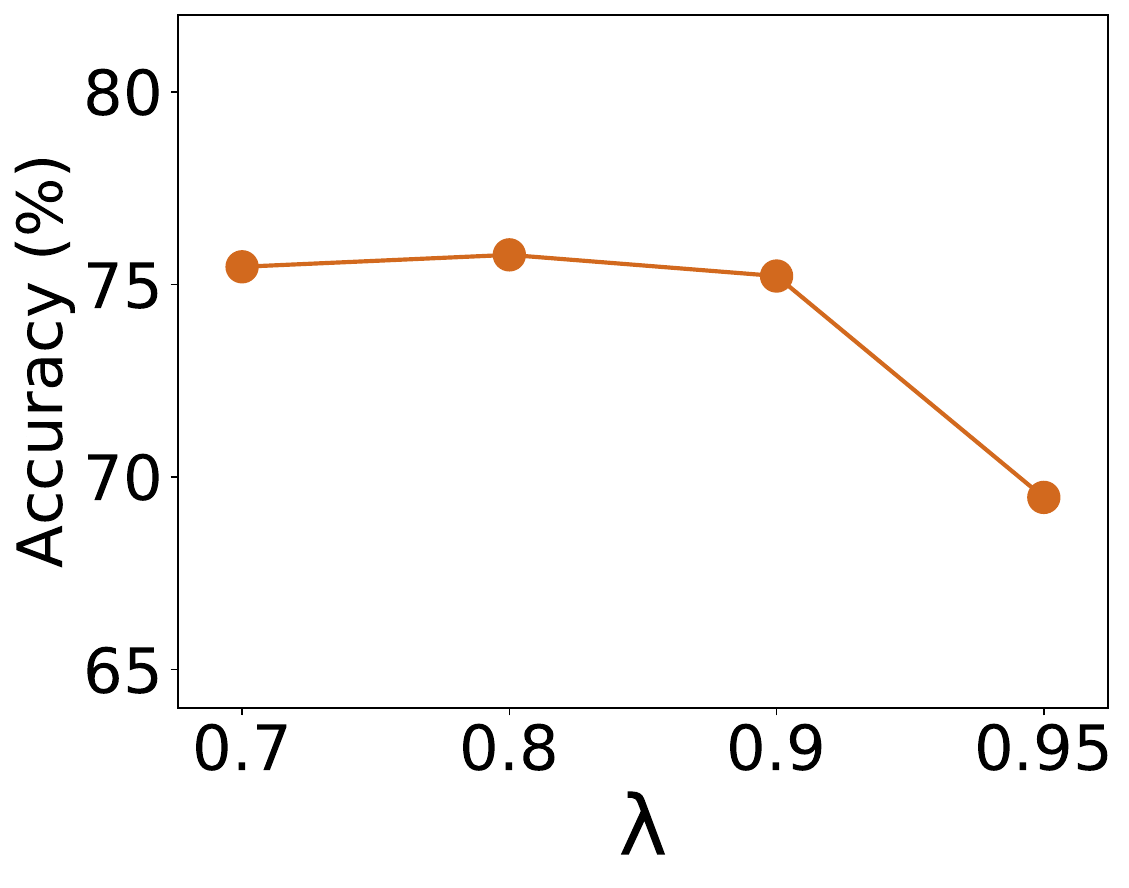}}
    
    \caption{Impact of the warm-up phase of $\lambda$ epochs on test classification accuracy on CIFAR-10 (C10) and CIFAR-100 (C100) under 40\% symmetric and asymmetric label noise.}
    \label{fig_lambda}
\end{figure}

\section{Conclusion}\label{sec6}
We propose Adaptive Label Refinement (ALR) to address the challenge of label noise in classification tasks. ALR employs temporal ensembling to transform original noisy labels into soft labels and introduces an entropy loss regularization term to adaptively enhance the learning of clean samples. ALR effectively alleviates the negative impact of noisy labels while also reducing underfitting on clean samples. Extensive experiments on both synthetic and real-world datasets demonstrate its robustness and effectiveness across diverse scenarios. Unlike conventional methods that rely on explicitly distinguishing clean samples from mislabeled ones, ALR treats both types of samples in a unified framework. This characteristic makes ALR a straightforward, versatile, and practical solution for addressing label noise challenges.

\section*{CRediT authorship contribution statement}
\textbf{Wenzhen Zhang}: Visualization, Writing – original draft, Methodology, Software, Investigation. \textbf{Debo Cheng}: Writing – review \& editing, methodology, Project administration, Conceptualization. \textbf{Guangquan Lu}: Data curation, Resources, Validation. \textbf{Bo Zhou}: Validation, Formal analysis. \textbf{Jiaye Li}: Investigation, Visualization, Software. \textbf{Shichao Zhang}: Writing – review \& editing, Supervision, Funding acquisition.

\section*{Declaration of competing interest}
The authors declare that they have no known competing financial interests or personal relationships that could have appeared to influence the work reported in this paper.

\section*{Acknowledgments}
This work was supported in part by the Project of Guangxi Science and Technology (GuiKeAB23026040), Innovation Project of Guangxi Graduate Education (YCBZ2023061), Research Fund of Guangxi Key Lab of Multi-source Information Mining \& Security (24-A-01-02).

\section*{Data availability}
Data will be made available on request.

\bibliographystyle{elsarticle-harv} 
\bibliography{elsarticle}







\end{document}